\documentclass[twocolumn]{autart}    % Enable this line and disable the preceding line to obtain a two-column  document whose style resembles the printed Automatica style.
\usepackage{hyperref}       % hyperlinks
\usepackage{url}            % simple URL typesetting
\usepackage{booktabs}       % professional-quality tables
\usepackage{nicefrac}       % compact symbols for 1/2, etc.
\usepackage{color}
\usepackage{xcolor}
\usepackage{amssymb, amsfonts, mathtools, enumitem, amsbsy}
\usepackage{pifont}
\usepackage{float}
\usepackage{subfigure}
\newcommand{\R}{\mathbb{R}}

\newcommand{\dn}{\mathop{\mathrm{dn}}}
\newcommand{\up}{\mathop{\mathrm{up}}}

\newcommand{\Mnorm}[1]{{\left\vert\kern-0.30ex\left\vert\kern-0.30ex\left\vert #1 
		\right\vert\kern-0.30ex\right\vert\kern-0.30ex\right\vert}}

\newcount\Comments  % 0 suppresses notes to selves in text
\definecolor{darkred}{rgb}{0.7,0,0}
\definecolor{drkgreen}{rgb}{0,0.5,0}
\definecolor{purple}{rgb}{0.6,0.0,0.8}
\newcommand{\kibitz}[2]{\ifnum\Comments=1{\textcolor{#1}{\textsf{\footnotesize #2}}}\fi}

\newcommand{\aditya}[1]{\kibitz{drkgreen}{[AM: #1]}}

\newcommand{\kazem}[1]{\kibitz{red}{[MF: #1]}}
\renewcommand{\hat}[1]{\widehat{#1}}

\newtheorem{theorem}{Theorem}
\newtheorem{proposition}{Proposition}
\newtheorem{lemma}{Lemma}
\newtheorem{assump}{Assumption}
\newtheorem{definition}{Definition}

\usepackage{aut_math_definitions}

\renewcommand{\underbar}[1]{\mkern 1.5mu\underline{\mkern-1.5mu#1\mkern-1.5mu}\mkern 1.5mu}

\newcommand{\MF}[1]{\kazem{#1}}

\begin{document}

\begin{frontmatter}
% papers but only if the title is over 5 words. Running title is not shown in output.
\title{Joint Learning of Linear Time-Invariant Dynamical Systems} % Title, preferably not more than 10 words.

\thanks[footnoteinfo]{This paper was not presented at any IFAC meeting.}

\author[modi]{Aditya Modi}\ead{admodi@umich.edu},    % Add the 
\author[kazem]{Mohamad Kazem Shirani Faradonbeh}\ead{mkshiranyf@gmail.com},               % e-mail address 
\author[ambuj]{Ambuj Tewari}\ead{tewaria@umich.edu},  % (ead) as shown
\author[george]{George Michailidis}\ead{gmichail@ucla.edu}

\address[modi]{Microsoft, Mountain View, CA-94043, USA}  % Please supply                                              
\address[kazem]{Department of Mathematics, Southern Methodist University, Dallas, TX-75275, USA}             % full addresses
\address[ambuj]{Department of Statistics, University of Michigan, Ann Arbor, MI-48109, USA}        % here.
\address[george]{Department of Statistics, University of California, Los Angeles, CA-90095, USA}

\begin{keyword}
Multiple Linear Systems; Data Sharing; Finite Time Identification; Autoregressive Processes;
Joint Estimation.
\end{keyword} 

\begin{abstract}
Linear time-invariant systems are very popular models in system theory and applications. 
A fundamental problem in system identification that remains rather unaddressed in extant literature is to leverage commonalities amongst \textit{related} systems to estimate their transition matrices more accurately. 
To address this problem, we investigate methods for jointly estimating the transition matrices of multiple systems. It is assumed that the transition matrices are \emph{unknown} linear functions of some \emph{unknown} shared basis matrices. 
We establish finite-time estimation error rates that fully reflect the roles of {trajectory lengths}, {dimension}, and {number of systems} under consideration.
The presented results are fairly general and show the significant gains that can be achieved by pooling data across systems, in comparison to learning each system individually. Further, they are shown to be \emph{robust} against moderate model misspecifications. 
To obtain the results, we develop novel techniques that are of independent interest and are applicable to similar problems. 
They include tightly bounding estimation errors in terms of the eigen-structures of transition matrices, establishing sharp high probability bounds for singular values of dependent random matrices, and capturing effects of misspecified transition matrices as the systems evolve over time.
\end{abstract}
\end{frontmatter}

\section{Introduction}
The problem of identifying the transition matrices in linear time-invariant (LTI) systems has been extensively studied in the literature \cite{buchmann2007asymptotic,kailath2000linear,lai1983asymptotic}. Recent papers establish finite-time rates for accurately learning the dynamics in various online and offline settings~\cite{faradonbeh2018finite,sarkar2019near,simchowitz2018learning}. Notably, existing results are established when the goal is to identify the transition matrix of \emph{a single} system. 

However, in many application areas of LTI systems, one observes state trajectories of \emph{multiple} dynamical systems. So, in order to be able to efficiently use the full data of all state trajectories and utilize the possible commonalities the systems share, we need to estimate the transition matrices of all systems \emph{jointly}. The range of applications is remarkably extensive, including dynamics of economic indicators in US states \cite{pesaran2015time,skripnikov2019joint,stock2016dynamic}, flight dynamics of airplanes at different altitudes \cite{bosworth1992linearized}, drivers of gene expressions across related species \cite{basu2015network,fujita2007modeling}, time series data of multiple subjects that suffer from the same disease \cite{seth2015granger,skripnikov2019regularized}, and commonalities among multiple subsystems in control engineering \cite{Sudhakara2022scalable}.  

%Further, the underlying dynamical systems share commonalities, but also exhibit heterogeneity. For example, \citep{skripnikov2019joint} analyze economic indicators of US states whose local economies share a strong manufacturing base. Moreover, in time course genetics experiments, one is interested in understanding the dynamics and drivers of gene expressions across related animal or cell line populations \citep{basu2015network}, while in neuroimaging, one has access to data from multiple subjects that suffer from the same disease \citep{skripnikov2019regularized}. 

In all these settings, there are strong similarities in the dynamics of the systems, which are unknown and need to be learned from the data. %but can also exhibit some degree of heterogeneity.
Hence, it becomes of interest to develop a joint learning strategy for the system parameters, by pooling the data of the underlying systems together and learn the \emph{unknown} similarities in their dynamics. In particular, this strategy is of extra importance in settings wherein the available data is limited, for example when the state trajectories are short or the dimensions are not small.

In general, joint learning (also referred to as multitask learning) approaches aim to study estimation methods subject to \emph{unknown} similarities across the data generation mechanisms. 
%In addition, it is of interest to estimate the idiosyncratic component pertaining to each system. \kazem{what is the point of the last sentence?}
Joint learning methods are studied %considered and their theoretical guarantees are established 
in supervised learning and online settings \cite{caruana1997multitask,ando2005framework,maurer2006bounds,maurer2016benefit,alquier2017regret}. Their theoretical analyses are obtained rely on a number of technical assumptions regarding the data, including independence, identical distributions, boundedness, richness, and isotropy. 

However, for the problem of joint learning of dynamical systems, additional technical challenges are present. First, the observations are temporally dependent. Second, the number of unknown parameters is the \textit{square} of the dimension of the system, which impacts the learning accuracy.
Third, since in many applications the dynamics matrices of the underlying LTI systems might possess eigenvalues of (almost) unit magnitude, conventional approaches for dependent data (e.g., mixing) inapplicable \cite{faradonbeh2018finite,sarkar2019near,simchowitz2018learning}. Fourth, the spectral properties of the transition matrices play a critical role on the magnitude of the estimation errors. Technically, the state vectors of the systems can scale exponentially with the multiplicities of the eigenvalues of the transition matrices (which can be as large as the dimension). Accordingly, novel techniques are required for considering all important factors and new analytical tools are needed for establishing useful rates for estimation error. Further details and technical discussions are provided in \pref{sec:joint-learning}.

We focus on a commonly used setting for joint learning that involves \emph{two layers of uncertainties}. It lets all systems share a common basis, while coefficients of the linear combinations are \emph{idiosyncratic} for each system. Such settings are adopted in multitask regression, linear bandits, and Markov decision processes \cite{du2020few,hu2021near,lu2021power,tripuraneni2020provable}. From another point of view, this assumption that the system transition matrices are \emph{unknown} linear combinations of \emph{unknown} basis matrices can be considered as a first-order approximation for {unknown} non-linear dynamical systems \cite{kang1993approximate,li2004iterative}.
Further, these compound layers of uncertainties subsume a recently studied case for mixtures of LTI systems where under additional assumptions such as exponential stability and distinguishable transition matrices, joint learning from unlabeled state trajectories outperforms individual system identification \cite{chen2022learning}.

The main contributions of this work can be summarized as follows. We provide novel finite-time estimation error bounds for jointly learning multiple systems, and establish that pooling the data of state trajectories can drastically decrease the estimation error. Our analysis also presents effects of different parameters on estimation accuracy, including dimension, spectral radius, eigenvalues multiplicity, tail properties of the noise processes, and heterogeneity among the systems. Further, we study learning accuracy in the presence of model misspecifications and show that the developed joint estimator can robustly handle moderate violations of the shared structure in the dynamics matrices.

In order to obtain the results, we employ advanced techniques from random matrix theory and prove sharp concentration results for sums of multiple dependent random matrices. Then, we establish tight and simultaneous high-probability confidence bounds for the sample covariance matrices of the systems under study. The analyses precisely characterize the dependence of the presented bounds on the spectral properties of the transition matrices, condition numbers, and block-sizes in the Jordan decomposition. Further, to address the issue of temporal dependence, we extend self-normalized martingale bounds to {multiple matrix-valued martingales}, subject to shared structures across the systems. We also present a robustness result by showing that the error due to misspecifications can be effectively controlled.

The remainder of the paper is organized as follows. The problem is formulated in \pref{sec:formulation}. In \pref{sec:joint-learning}, we describe the joint-learning procedure, study the per-system estimation error, and provide the roles of 
various key quantities. Then, investigation of robustness to model misspecification and the impact of violating the shared structure are discussed in \pref{sec:misspec}. We provide numerical illustrations for joint learning in \pref{sec:numerical} and present the proofs of our results in the subsequent sections. Finally, the paper is concluded in \pref{sec:conc}. 

\textbf{Notation.} For a matrix $A$, $A'$ denotes the transpose of $A$. For square matrices, we use
the following order of eigenvalues in terms of their magnitudes: $\abr{\lambda_{\max}(A)} = \abr{\lambda_1(A)} \ge \abr{\lambda_2(A)} \ge \cdots \ge \abr{\lambda_d(A)} = \abr{\lambda_{\min}(A)}$. For singular values, we employ $\sigma_{\min}(A)$ and $\sigma_{\max}(A)$. For any vector $v \in \CC^d$, let $\norm{v}_p$ denote its $\ell_p$ norm. We use $\Mnorm{\cdot}_{\gamma \rightarrow \beta}$ to denote the matrix operator-norm for $\beta, \gamma \in [1,\infty]$ and $A \in \CC^{d_1 \times d_2}$: $\opnorm{A}{\gamma \rightarrow \beta} = \sup\limits_{v \neq \zero } {\norm{Av}_\beta}/{\norm{v}_{\gamma}}$. When $\gamma=\beta$, we simply write $\Mnorm{A}_\beta$. For functions $f,g: \Xcal \rightarrow \R$, we write $f \lesssim g$, if $f(x) \le c g(x)$ for a universal constant $c>0$. Similarly, we use $f = O(g)$ and $f = \Omega(h)$, if $0 \le f(n) \le c_1 g(n)$ for all $n \ge n_1$, and $0 \le c_2 h(n) \le f(n)$ for all $n \ge n_2$, respectively, where $c_1, c_2,n_1,n_2$ are large enough constants. For any two matrices of the same dimensions, we define the inner product $\inner{A}{B} = \tr{A'B}$. Then, the Frobenius norm becomes $\norm{A}_F = \sqrt{\inner{A}{A}}$. The sigma-field generated by $X_1, X_2, \ldots, X_n$ is denoted by $\sigma(X_1, X_2,\ldots X_n)$. We denote the $i$-th component of the vector $x \in \R^d$ by $x[i]$. Finally, for $n \in \NN$, the shorthand $[n]$ is the set $\cbr{1,2,\ldots,n}$. 

\section{Problem Formulation}\label{sec:formulation}
\begin{figure}[ht]
\centering     %%% not \center
\subfigure[{$\abr{\lambda_{1}(A)}<1$}]{\label{fig:a}\includegraphics[width=0.75\columnwidth]{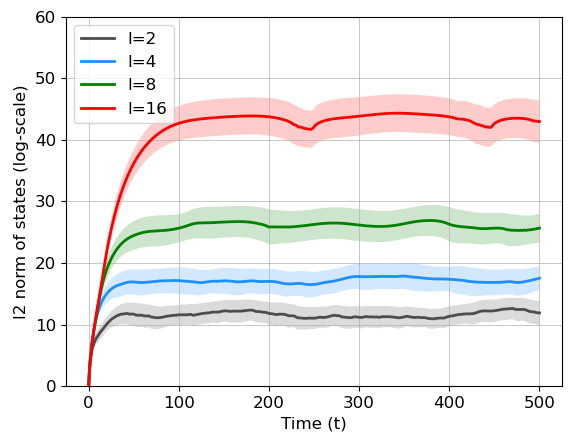}}
\subfigure[{$\abr{\lambda_{1}(A)}\approx1$}]{\label{fig:b}\includegraphics[width=0.76\columnwidth]{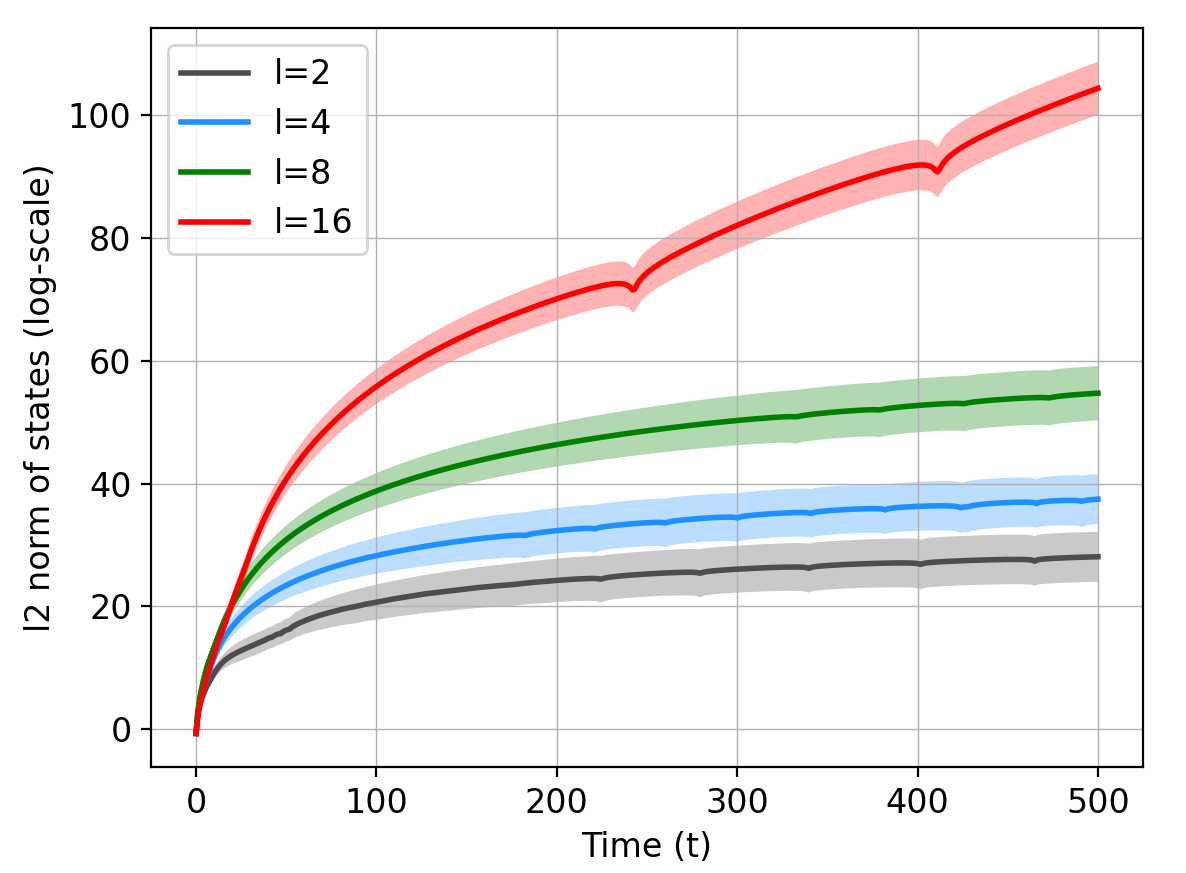}}
\caption{Logarithm of the magnitude of the state vectors vs. time, for different block-sizes in the Jordan forms of the transition matrices, which is denoted by $l$ in \pref{eq:JordanDef}. The exponential scaling of the state vectors with $l$ can be seen in both plots.}
\label{fig:state-size}
\end{figure}

Our main goal is to study the rates of jointly learning dynamics of multiple LTI systems. Data consists of state trajectories of length $T$ from $M$ different systems. Specifically, for $m \in [M]$ and $t = 0,1,\ldots, T$, let $x_m(t) \in \R^{d}$ denote the state of the $m$-th system, that evolves according to the Vector Auto-Regressive (VAR) process
\begin{align}
    \label{eq:mt-lti}
    x_m(t+1) = A_m x_m(t) + \eta_m(t+1).
\end{align}
Above, $A_m \in \R^{d \times d}$ denotes the true unknown transition matrix of the $m$-th system and $\eta_m(t+1)$ is a mean zero noise. For succinctness, we use $\Theta^*$ to denote the set of all $M$ transition matrices $\{A_m\}_{m=1}^M$.
The transition matrices are \textit{related} as will be specified in \pref{assum:linear_model}.

Note that the above setting includes systems with longer memories. Indeed, if the states $\tilde{x}_m(t) \in \R^{\tilde{d}}$ obey
$$\tilde{x}_m(t) = B_{m,1} \tilde{x}_m(t-1) + \cdots + B_{m,q}\tilde{x}_m(t-q)+ \eta_m(t),$$
then, by concatenating $\tilde{x}_m(t-1), \cdots, \tilde{x}_m(t-q)$ in one larger vector $x_m(t-1)$, the new state dynamics is \pref{eq:mt-lti}, for $d =q\tilde{d}$ and $A_m = \begin{bmatrix}
B_{m,1} \cdots B_{m,q-1} & B_{m,q}\\
I_{(q-1)\tilde{d}} & 0
\end{bmatrix}.$ 

We assume that the system states do not explode in the sense that the spectral radius of the transition matrix $A_m$ can be \emph{slightly} larger than one. This is required for the systems to be able to operate for a reasonable time length~\cite{juselius2002high,faradonbeh2018bfinite}. Note that this assumption still lets the state vectors grow with time, as shown in \pref{fig:state-size}.
\begin{assump}
\label{assum:non-explosive}
For all $m \in [M]$, we have $\abr{\lambda_1(A_m)} \le 1+\rho/T$, where $\rho>0$ is a fixed constant.
\end{assump}

In addition to the magnitudes of the eigenvalues, further properties of the transition matrices heavily determine the temporal evolution of the systems. A very important one is the size of the largest block in the Jordan decomposition of $A_m$, which will be rigorously defined shortly. 
This quantity is denoted by $l$ in \pref{eq:JordanDef}. The impact of $l$ on the state trajectories is illustrated in \pref{fig:state-size}, wherein we plot the \emph{logarithm} of the magnitude of state vectors for linear systems of dimension $d=32$. The upper plot depicts state magnitude for stable systems and for blocks of the size $l=2,4,8,16$ in the Jordan decomposition of the transition matrices. It illustrates that the state vector scales \emph{exponentially} with $l$. Note that $l$ can be as large as the system dimension $d$.

Moreover, the case of transition matrices with eigenvalues close to (or exactly on) the unit circle is provided in the lower panel in \pref{fig:state-size}. It illustrates that the state vectors grow polynomially with time, whereas the scaling with the block-size $l$ is exponential. Therefore, in design and analysis of joint learning methods, one needs to carefully consider the effects of $l$ and $\abr{\lambda_1(A_m)}$. 
%To that end, we define $\alpha\left(A_m\right)$ in \pref{eq:alpha_def} and use it in our theoretical analyses that will be discussed later on. 

\iffalse
To describe the assumption on the noise sequence for each system, we first define the sub-Gaussian norm of a random variable.
\begin{definition}[\cite{vershynin2018high}]
\label{def:subg-norm}
For a sub-Gaussian random variable $X$, the sub-Gaussian norm $\norm{X}_{\psi_2}$ is defined as 
\begin{align*}
    \norm{X}_{\psi_2} = \inf \{t > 0: \EE \exp(X^2/t^2) \le 2\}.
\end{align*}
For such a random variable, equivalently we also have: $\PP[|X| \ge t] \le 2 \exp(-ct^2/\norm{X}_{\psi_2}^2)$ and $\EE[X^2] \le c'\norm{X}_{\psi_2}^2$ where $c,c'$ are absolute constants. If $\EE[X] = 0$, then $\EE\sbr{\exp(\lambda X)} \le \exp(c'\lambda^2\norm{X}_{\psi_2}^2)$.
\end{definition}
\fi

Next, we express the probabilistic properties of the stochastic processes driving the dynamical systems. Let $\Fcal_t = \sigma(x_{1:M}(0),\eta_{1:M}(1),\dots,\eta_{1:M}(t))$ denote the filtration generated by the the initial state and the sequence of noise vectors. Based on this, we adopt the following ubiquitous setting that lets the noise process $\{\eta_m(t)\}_{t=1}^\infty$ be a sub-Gaussian martingale difference sequence. Note that by definition, $\eta_m(t)$ is $\Fcal_{t}$-measurable.
\begin{assump}
\label{assum:noise}
For all systems $m \in [M]$, we have $\EE\sbr{\eta_m(t)|\Fcal_{t-1}} = \zero$ and $\EE\sbr{\eta_m(t) \eta_m(t)'|\Fcal_{t-1}} = C$. Further, $\eta_m(t)$ is sub-Gaussian; for all $\lambda \in \R^d$:
$$\EE\sbr{\exp \inner{\lambda}{\eta_m(t)} | \Fcal_{t-1}} \le \exp \left( \norm{\lambda}^2 \sigma^2/2 \right).$$
Henceforth, we denote $c^2=\max(\sigma^2, \lambda_{\max}(C))$. 
\end{assump}
The above assumption is widely-used in the finite-sample analysis of statistical learning methods~\cite{abbasi2011improved,faradonbeh2020input}. It includes normally distributed martingale difference sequences, for which \pref{assum:noise} is satisfied with $\sigma^2 = \lambda_{\max}(C)$. Moreover, if the coordinates of $\eta_m(t)$ are (conditionally) independent and have sub-Gaussian distributions with constant $\sigma_i$, it suffices to let $\sigma^2 = \sum_{i=1}^d \sigma_i^2$. {We let a common noise covariance matrix for the ease of expression. However, the results simply generalize to covariance matrices that vary with time and across the systems, by appropriately replacing upper- and lower-bounds of the matrices~\cite{faradonbeh2018finite,sarkar2019near,simchowitz2018learning}.}

For a single system $m \in [M]$, its underlying transition matrices $A_m$ can be \emph{individually} learned from its own state trajectory data by using the least squares estimator \cite{faradonbeh2018finite,sarkar2019near}. We are interested in jointly learning the transition matrices of all $M$ systems under the assumption that they share the following common structure. 
\begin{assump}[Shared Basis]
\label{assum:linear_model}
Each transition matrix $A_m$ can be expressed as
\begin{align}
\label{eq:linear_model}
    A_m = \sum_{i=1}^k \beta^*_{m}[i] W^*_i,
\end{align}
where $\{W^*_i\}_{i=1}^k$ are common ${d \times d}$ matrices and $\beta^*_m \in \R^k$ contains the idiosyncratic coefficients for system $m$.
\end{assump}

This assumption is commonly-used in the literature of jointly learning multiple parameters~\cite{du2020few,tripuraneni2020provable}. Intuitively, it states that each system evolves by combining the effects of $k$ systems. These $k$ unknown systems behind the scene are shared by all systems $m \in [M]$, the weight of each of which is reflected by the idiosyncratic coefficients that are collected in $\beta^*_m$ for system $m$. Thereby, the model allows for a rich heterogeneity across systems. 

The main goal is to estimate $\Theta^*=\{A_m\}_{m=1}^M$ by observing $x_m(t)$ for $1 \leq m \leq M$ and $0 \leq t \leq T$. To that end, we need a reliable joint estimator that can leverage the unknown shared structure to learn from the state trajectories more accurately than individual estimations of the dynamics. Importantly, to theoretically analyze effects of all quantities on the estimation error, we encounter some challenges for joint learning of multiple systems that do \emph{not} appear in single-system identification. 

Technically, the least-squares estimate of the transition matrix of a single system admits a closed form that lets the main challenge of the analysis be concentration of the sample covariance matrix of the state vectors. However, since closed forms are not achievable for joint-estimators, learning accuracy cannot be directly analyzed. To address this, we first bound the prediction error and then use that for bounding the estimation error. To establish the former, after appropriately decomposing the joint prediction error, we study its scaling with the trajectory-length and dimension, as well as the trade-offs between the number of systems, number of basis matrices, and magnitudes of the state vectors. Then, we deconvolve the prediction error to the estimation error and the sample covariance matrices, and show useful bounds that can tightly relate the largest and smallest eigenvalues of the sample covariance matrices across all systems. Notably, this step that is not required in single-system identification is based on novel probabilistic analysis for dependent random matrices.

In the sequel, we introduce a joint estimator for utilizing the structure in \pref{assum:linear_model} and analyze its accuracy. Then, in \pref{sec:misspec} we consider violations of the structure in \pref{eq:linear_model} and establish robustness guarantees.% by allowing idiosyncratic additive factors for each system $m \in [M]$. 

\section{Joint Learning of LTI Systems}
\label{sec:joint-learning}
In this section, we propose an estimator for jointly learning the $M$ transition matrices. Then, we establish that the estimation error decays at a significantly faster rate than competing procedures that learn each transition matrix $A_m$ separately by using only the data trajectory of system $m$.

Based on the parameterization in \pref{eq:linear_model}, we solve for $\widehat{\Wb} = \{\widehat{W}_i\}_{i=1}^k$ and $\widehat{B} = \left[ \hat \beta_1 | \hat \beta_2 | \cdots \hat \beta_M \right] \in \R^{k \times M}$, as follows:
\begin{align}
    \widehat{\Wb}, \hat B \coloneqq & \argmin_{\Wb, B} \Lcal(\Theta^*, \Wb, B), \label{eq:mt_opt}
\end{align}
where $\Lcal(\Theta^*, \Wb, B)$ is the averaged squared loss across all $M$ systems:
\begin{align*} 
    \frac{1}{MT}{\sum_{m=1}^M \sum_{t=0}^T \norm{x_m(t+1) - \rbr{\sum_{i=1}^k \beta_{m}[i] W_i} x_m(t)}_2^2 }.
\end{align*}

In the analysis, we assume that one can approximately find the minimizer in \pref{eq:mt_opt}. Although the loss function in \pref{eq:mt_opt} is non-convex, thanks to its structure, computationally fast methods for accurately finding the minimizer are applicable. {Specifically, the loss function in \pref{eq:mt_opt} is quadratic and the non-convexity is the bilinear dependence on $(\Wb, B)$}. The optimization in \pref{eq:mt_opt} is of the form of explicit rank-constrained representations ~\cite{burer2003nonlinear}. For such problems, it has been shown under mild conditions that gradient descent converges to a low-rank minimizer at a linear rate~\cite{wang2017unified}. {Moreover, it is known that methods such as stochastic gradient descent have global convergence, and these bilinear non-convexities do not lead to any spurious local minima \cite{ge2017no}. In addition, since the loss function is biconvex in $\Wb$ and $B$, alternating minimization techniques converge to global optima, under standard assumptions \cite{jain2017non}. Nonetheless, note that a near-optimal minimum for the objective function is sufficient, and we only need to estimate the product $W B$ accurately instead of recovering both $W$ and $B$}. More specifically, the error of the joint estimator in \pref{eq:mt_opt} degrades gracefully in the presence of moderate optimization errors. For instance, suppose that the optimization problem is solved up to an error of $\epsilon$ from a global optimum. It can be shown that an additional term of magnitude $O\rbr{\epsilon/\lambda_{\min}(C)}$ arises in the estimation error, due to this optimization error. Numerical experiments in \pref{sec:numerical}  illustrate the implementation of \pref{eq:mt_opt}.

In the sequel, we provide key results for the joint estimator in \pref{eq:mt_opt} and establish the high probability decay rates of $\sum_{m=1}^M \norm{A_m - \hat{A}_m}_F^2$. 

The analysis leverages high probability bounds on the sample covariance matrices of all systems, denoted by
\begin{equation*}
    \Sigma_m = \sum_{t=0}^{T-1} x_m(t)x_m(t)'.
\end{equation*}
For that purpose, we utilize the Jordan forms of matrices, as follows. For matrix $A_m$, its Jordan decomposition is $A_m = P^{-1}_m \Lambda_m P_m$, where $\Lambda_m$ is a block diagonal matrix; $\Lambda_m = \diag(\Lambda_{m,1},\ldots\Lambda_{m,q_m})$, and for $i=1,\ldots q_m$, each block $\Lambda_{m,i} \in \CC^{l_{m,i} \times l_{m,i}}$ is a Jordan matrix of the eigenvalue $\lambda_{m,i}$. A Jordan matrix of size $l$ for $\lambda \in \CC$ is 
\begin{align} \label{eq:JordanDef}
    \begin{bmatrix}
    \lambda & 1 & 0 & \ldots & 0 & 0 \\
    0 & \lambda & 1 & 0 & \ldots & 0 \\
    \vdots & \vdots & \vdots & \vdots & \vdots & \vdots \\
    0 & 0 & 0 & \ldots & 0 & \lambda
    \end{bmatrix} \in \CC^{l \times l}.
\end{align}
Henceforth, we denote the size of each Jordan block by $l_{m,i}$, for $i=1,\cdots,q_m$, and the size of the largest Jordan block for system $m$ by $l^*_m$. Note that for \emph{diagonalizable} matrices $A_m$, since $\Lambda_m$ is diagonal, we have $l^*_m=1$. {Now, using this notation, we define
{\begin{align}
\label{eq:alpha_def}
    \alpha(A_m) = \begin{cases} \opnorm{P_m^{-1}}{\infty\rightarrow 2}\opnorm{P_m}{\infty} f(\Lambda_m) &\abr{\lambda_{m,1}} < 1-\frac{\rho}{T} \\
    \opnorm{P_m^{-1}}{\infty\rightarrow 2}\opnorm{P_m}{\infty} e^{\rho+1} & \abr{\abr{\lambda_{m,1}}-1} \le \frac{\rho}{T},\end{cases}
\end{align}}}
where $\lambda_{m,1}=\lambda_1 \left( A_m \right)$ and
$$f(\Lambda_m) = e^{1/|\lambda_{m,1}|} \sbr{\frac{l^*_m-1}{- \log |\lambda_{m,1}|} + \frac{(l^*_m-1)!}{(-\log |\lambda_{m,1}|)^{l^*_{m}}}}.$$ 
The quantities in the definition of $\alpha\left(A_m\right)$ can be interpreted as follows. {The term $\opnorm{P_m^{-1}}{\infty\rightarrow 2}\opnorm{P_m}{\infty}$ is similar to the condition number of the similarity matrix $P_m$ in the Jordan decomposition that is used to block-diagonalize the matrix.} 
Moreover, $f\left(\Lambda_m\right)$ for stable matrices, and $e^{\rho+1}$ for transition matrices with (almost) unit eigenvalues, capture the \emph{long term} influences of the eigenvalues. In other words, $f\left(\Lambda_m\right)$ indicates the amount that $\eta_m(t)$ contributes to the growth of $\norm{x_m(s)}$, for $s \gg t$ and $\abr{\lambda_{m,1}}<1-{\rho}/{T}$. When $|\lambda| \approx 1$, $\norm{x_m(s)}$ scales polynomially with the trajectory length $T$, since influences of the noise vectors $\eta_m(t)$ do not decay as $s-t$ grows, because of the accumulations caused by the unit eigenvalues. The exact expressions are in \pref{thm:cov_conc} below. {Note that while $f(\Lambda_m)$ is used to obtain an analytical upper bound for the whole range $\abr{\lambda_{m,1}} < 1-{\rho}/{T}$, it is not tight for small values of $\lambda_{m,1}$ and tighter expressions can be obtained using the analysis in the proof of \pref{thm:cov_conc}.}

To introduce the following result, we define $\bar b_m$ next. First, for some $\delta_C>0$ that will be determined later, for system $m$, define $\bar b_m = b_T(\delta_C/3) + \norm{x_m(0)}_\infty$, where $b_T(\delta)=\sqrt{2\sigma^2 \log \left(2dMT \delta^{-1}\right)}$. Then, we establish high probability bounds on the sample covariance matrices $\Sigma_m$ with the detailed proof provided in \pref{sec:cov_conc_proof}.

\begin{theorem}[Covariance matrices]
\label{thm:cov_conc}
{Under Assumptions \ref{assum:non-explosive} and \ref{assum:noise}}, for each system $m$, let $\underbar \Sigma_m=\underbar \lambda_m I$ and  $\bar \Sigma_m = \bar \lambda_m I$, where  $\underbar\lambda_{m} \coloneqq 4^{-1}\lambda_{\min}(C)T$,
and
\begin{align*}
    \bar{\lambda}_m \coloneqq \begin{cases}
    \alpha(A_m)^2 \bar b_m^2 T ,& \text{if } \abr{\lambda_{m,1}} < 1-\frac{\rho}{T},\\
    \alpha(A_m)^2 \bar b_m^2 T^{2l^*_m + 1},   & \text{if } \abr{\abr{\lambda_{m,1}}-1} \le \frac{\rho}{T}.
\end{cases}
\end{align*}
Then, there is $T_0$, such that for $m \in [M]$ and $T \ge T_0$: 
\begin{align}
    \PP\sbr{0 \prec \underbar{\Sigma}_m \preceq \Sigma_m \preceq \bar{\Sigma}_m} \ge 1-\delta_C.
\end{align}
\end{theorem}
The above two expressions for $\bar \lambda_m$ show that for $\abr{\lambda_{m,1}}<1-{\rho}/{T}$, the largest eigenvalue of the covariance matrix grows linearly in $T$, whereas for $\abr{\abr{\lambda_{m,1}}-1} \le {\rho}/{T}$, the bounds scale exponentially with the multiplicities of the eigenvalues. Note that the bounds in \pref{thm:cov_conc} and the estimation error results stated hereafter require the trajectories for each system to be longer than $T_0$. The precise definition for $T_0$ can be found in the statement of \pref{lem:cov_lowbnd} in \pref{sec:cov_conc_proof}.

For establishing the above, we extend existing tools for learning linear systems~\cite{abbasi2011improved,faradonbeh2018finite,sarkar2019near,vershynin2018high}. Specifically, we leverage truncation-based arguments and introduce the quantity $\alpha(A_m)$ that captures the effect of the spectral properties of the transition matrices on the magnitudes of the state trajectories. Further, we develop strategies for finding high probability bounds for largest and smallest singular values of random matrices and for studying self-normalized matrix-valued martingales.

Importantly, \pref{thm:cov_conc} provides a tight characterization of the sample covariance matrix for each system, in terms of the magnitudes of eigenvalues of $A_m$, as well as the largest block-size in the Jordan decomposition of $A_m$. The upper bounds show that $\bar \lambda_m$ grows exponentially with the dimension $d$, whenever $l^*_m = \Omega(d)$. Further, if $A_m$ has eigenvalues with magnitudes close to $1$, then scaling with time $T$ can be as large as $T^{2d+1}$. The bounds in \pref{thm:cov_conc} are more general than $\tr{\sum_{t=0}^T A_m^t A'_m{}^{t}}$ that appears in some analyses \cite{sarkar2019near,simchowitz2018learning}, and can be used to calculate the latter term. Finally, \pref{thm:cov_conc} indicates that the classical framework of persistent excitation \cite{boyd1986necessary,green1986persistence,jenkins2018convergence} is not applicable, since the lower and upper bounds of eigenvalues grow at drastically different rates.

Next, we express the joint estimation error rates.  
\begin{definition}
    Denote $\Ecal_C=\big\{ 0 \prec \underbar{\Sigma}_m \preceq \Sigma_m \preceq \bar{\Sigma}_m \big\}$, and let  $\bar{\lambda} = \max_m \bar{\lambda}_m$, $\underbar{\lambda} = \min_m \underbar{\lambda}_m$, $\boldsymbol{\kappa}_m = \bar{\lambda}_m/\underbar{\lambda}_m$, $\boldsymbol\kappa = \max_m \boldsymbol\kappa_m$, and $\boldsymbol\kappa_\infty = \bar{\lambda}/\underbar{\lambda}$. Note that $\boldsymbol\kappa_\infty > \boldsymbol\kappa$.
\end{definition}

\begin{theorem}
    \label{thm:estimation_error}
{Under Assumptions \ref{assum:non-explosive}, \ref{assum:noise}, and \ref{assum:linear_model}, and for $T\ge T_0$}, the estimator in \pref{eq:mt_opt} returns $\hat A_m$ for each system $m \in [M]$, such that with probability at least $1-\delta$, the following holds:
\begin{align*}
    % \label{eq:est_error_bound}
    \frac{1}{M}\sum_{m=1}^M \norm{\hat A_m - A_m}_F^2 \lesssim \frac{c^2}{\underbar \lambda} \rbr{k \log \boldsymbol\kappa_\infty + \frac{d^2k}{M}\log \frac{\boldsymbol\kappa dT}{\delta}}.
\end{align*}
\end{theorem}

The proof is provided in \pref{sec:estimation-error-proof}. By putting Theorems \ref{thm:cov_conc} and \ref{thm:estimation_error} together, the estimation error per-system\footnote{{In order to obtain a guarantee for the maximum error over all systems, additional assumptions on the matrix $\left[ \beta^*_1  \ldots  \beta^*_M \right]$ are required. This problem falls beyond the scope of this paper and we leave it to a future work.}} is
\begin{align}
\label{eq:lti_mt_bound}
    %\frac{1}{M}\sum_{m=1}^M \norm{\hat A_m - A_m}_F^2 \lesssim 
    O\rbr{\frac{c^2 k \log \boldsymbol\kappa_\infty}{\lambda_{\min}(C)T} + \frac{c^2d^2k\log \frac{\boldsymbol\kappa dT}{\delta}}{M\lambda_{\min}(C)T}}.
\end{align}

The above expression demonstrates the effects of learning the systems in a joint manner. The first term in \pref{eq:lti_mt_bound} %$\frac{c^2 k \log \boldsymbol\kappa_\infty}{\lambda_{\min}(C)T}$ on the RHS 
can be interpreted as the error in estimating the idiosyncratic components $\beta_m$ for each system. The convergence rate is $O\rbr{{k}/{T}}$, as each $\beta_m$ is a $k$-dimensional parameter and for each system, we have a trajectory of length $T$. More importantly, the second term in \pref{eq:lti_mt_bound} % $\frac{c^2d^2k\log \frac{\boldsymbol\kappa dT}{\delta}}{M\lambda_{\min}(C)T}$ 
indicates that the joint estimator in \pref{eq:mt_opt} effectively increases the sample size for the shared components $\{W_i\}_{i=1}^k$, by pooling the data of all systems. So, the error decays as $O({d^2k}/{MT})$, showing that the effective sample size for $\{W_i\}_{i=1}^k$ is $MT$. 

In contrast, for individual learning of LTI systems, the rate is known \cite{faradonbeh2018finite,faradonbeh2018optimality,sarkar2019near,simchowitz2018learning} to be
\begin{align*}
    \norm{\hat A_m - A_m}_F^2 \lesssim \frac{c^2d^2 }{\lambda_{\min}(C)T} \log \frac{\alpha(A_m)T}{\delta}.
\end{align*}
{Thus, %when the largest block-sizes $l^*_m$ are not too large, 
the estimation error rate in \eqref{eq:lti_mt_bound} recovers the rate for a single system ($k=1$), and it significantly improves for joint learning}, especially when
\begin{align}
\label{eq:low-dim}
    k < d^2 ~~~~~~\text{ and }~~~~~~~ k < M.
\end{align}
%Note that if $l^*_m$ is large, the rates are again similar for both joint and individual learning methods (as the error is $O(\log\alpha(A_m))$). 
Note that the above conditions %in \pref{eq:low-dim} 
are as expected. First, when $k \approx d^2$, the structure in \pref{assum:linear_model} does \emph{not} provide any commonality among the systems. That is, for $k = d^2$, the LTI systems can be totally arbitrary and \pref{assum:linear_model} is automatically satisfied. This prevents reductions in the effective dimension of the unknown transition matrices, and also prevents joint learning from being any different than individual learning. Similarly, $k \approx M$ precludes all commonalities and indicates that $\{A_m\}_{m=1}^M$ are too heterogeneous to allow for any improved learning via joint estimation.

Importantly, when the largest block-size $l^*_m$ varies significantly across the $M$ systems, a higher degree of shared structure is needed to improve the joint estimation error for all systems.
Since $\boldsymbol \kappa$ and $\boldsymbol \kappa_\infty$ depend exponentially on $l^*_m$ (as shown in \pref{fig:state-size} and \pref{thm:cov_conc}) and $l^*_m$ can be as large as $d$, we can have $\log \boldsymbol \kappa_\infty = \log \boldsymbol \kappa = \Omega(d)$. Hence, in this situation we incur an additional dimension dependence in the error of the joint estimator. Note that such effects of $l^*_m$ are unavoidable (regardless of the employed estimator). Moreover, in this case, joint learning rates improve if $k \leq d \text{ and } kd \leq M$.
Therefore, our analysis highlights the important effects of the large blocks in the Jordan form of the transition matrices. 

The above is an inherent difference between estimating dynamics of LTI systems and learning from \emph{independent} observations. In fact, the analysis established in this work includes stochastic matrix regressions that the data of system $m$ consists of
\begin{align} \label{eq:stoch_reg}
    y_m(t) = A_m x_m(t) + \eta_m(t),
\end{align}
wherein the regressors $x_m(t)$ are drawn from some distribution $\Dcal_m$, and $y_m(t)$ is the response. Assume that $\left( x_m(t),y_m(t) \right)$ are independent as $m,t$ vary. Now, the sample covariance matrix $\Sigma_m$ for each system does not {depend} on $A_m$. Hence, the error for the joint estimator is not affected by the block-sizes in the Jordan decomposition of $A_m$. Therefore, in this setting, joint learning always leads to improved per-system error rates, as long as the necessary conditions $k < d^2$ and $k < M$ hold.

\section{Robustness to Misspecifications}
\label{sec:misspec}
In \pref{thm:estimation_error}, we showed that \pref{assum:linear_model} can be utilized for obtaining an improved estimation error, by jointly learning the $M$ systems. Next, we consider the impacts of misspecified models on the estimation error and study robustness of the proposed joint estimator against violations of the structure in \pref{assum:linear_model}. 

Let us first consider the deviation of the dynamics of each system $m \in [M]$ from the shared structure. Specifically, by employing the matrix $D_m$ to denote the deviation of system $m$ from \pref{assum:linear_model}, suppose that
\begin{align}
    \label{eq:linear_model_pert}
    A_m = \rbr{\sum_{i=1}^k \beta^*_m[i] W^*_i} + D_m. 
\end{align}
Then, denote the \emph{total {misspecification}} by $\bar{\zeta}^2=\sum_{m=1}^M \norm{D_m}_F^2$. We study the consequences of the above deviations, assuming that the same joint learning method as before is used for estimating the transition matrices. %We establish the robustness of the proposed estimation method.
%such that $\norm{D_m}_F \le \zeta_m$. Further, let $\bar{\zeta}^2=\sum_{m=1}^M \zeta_m^2$.
%\end{assump}
%Indeed, $D_m \in \R^{d \times d}$ is the deviation of system $m$ from the shared structure of linear combination and $\bar{\zeta}^2$ captures the total misspecification. Under \pref{assum:linear_model_pert}, 

\begin{theorem}
    \label{thm:estimation_error_pert}
%Under \pref{assum:linear_model_pert}, 
{Under Assumptions \ref{assum:non-explosive}, \ref{assum:noise}, \eqref{eq:linear_model_pert}, and for $T\ge T_0$}, the estimator in \pref{eq:mt_opt} returns $\hat A_m$ for each system $m \in [M]$, such that with probability at least $1-\delta$, we have:
\begin{align}
    & \frac{1}{M}\sum_{m=1}^M \norm{\hat A_m - A_m}_F^2 \lesssim \nonumber & \\
    & 
    {} \frac{c^2}{\underbar \lambda} \rbr{k \log \boldsymbol\kappa_\infty + \frac{d^2k}{M}\log \frac{\boldsymbol\kappa dT}{\delta}} + \frac{\rbr{\boldsymbol\kappa_\infty +1}\bar\zeta^2}{M}. \label{eq:est_error_bound_pert}
\end{align}
\end{theorem}

\iffalse
\aditya{I think we should simply remove the sketch for \pref{thm:estimation_error_pert}. It will save some space which we can use to write the introduction in a more understandable manner.}
\begin{proof}[Proof sketch for \pref{thm:estimation_error_pert}] For the misspecified case, we can again start by the fact that $\hat W, \hat B$ minimize the empirical loss, which implies:
\begin{align*}
    & \frac{1}{2}\norm{ \Xcal (W^*B^* - \hat W\hat B) }_2^2 \le {}  \inner{Z}{\Xcal \rbr{W^*B^* - \hat W \hat B}} & \\
    & \qquad  + \sum_{m=1}^M 2\inner{ \tilde{X}_m\tilde{D}_m}{ \tilde{X}_m \rbr{W^*B^* - \hat W \hat B}} &
\end{align*}
Thus, in addition to the previous terms in \pref{eq:pred_error_breakup}, we also need to handle the additional missspecification term in red. Using \pref{assum:linear_model_pert}, we can bound the additional term as:
\begin{align*}
    & \sum_{m=1}^M 2\inner{ \tilde{X}_m\tilde{D}_m}{ \tilde{X}_m \rbr{W^*B^* - \hat W \hat B}} &\\
    & \le  \sqrt{\sum_{m=1}^M \bar \lambda_m \zeta_m^2} \norm{\Xcal\rbr{W^*B^* - \hat W \hat B}}_F \\
    {} & \quad \le \sqrt{\bar \lambda \zeta^2} \norm{\Xcal\rbr{W^*B^* - \hat W \hat B}}_F
\end{align*}
We add this term to the prediction error breakup and accounting for the term in the analysis leads to the final estimation error bound under bounded misspecification. We refer the reader to \pref{app:estimation-error} for a detailed proof.
\end{proof}
\fi
% \MF{Is the denominator $\underline{\lambda}$? }
The proof of \pref{thm:estimation_error_pert} is provided in \pref{sec:estimation-error_pert_proof}.
In \pref{eq:est_error_bound_pert}, we observe that the total misspecification $\bar \zeta^2$ imposes an additional error of $(\boldsymbol\kappa_\infty+1) \bar \zeta^2$ for jointly learning all $M$ system. Hence, to obtain accurate estimates, we need the total misspecification $\bar \zeta^2$ to be smaller than the number of systems $M$, as one can expect. The discussion following \pref{thm:estimation_error} is still applicable in the misspecified setting and indicates that in order to have accurate estimates, the number of the shared bases $k$ must be smaller than $M$ as well. In addition, compared to individual learning, the joint estimation error improves \emph{despite the unknown model misspecifications}, as long as
$$\frac{\boldsymbol \kappa_\infty \bar \zeta^2}{M} \lesssim \frac{d^2}{T}.$$ This shows that when the total misspecification is proportional to the number of systems; $\bar \zeta^* = \Omega(M)$, we pay a constant factor proportional to $\boldsymbol \kappa_\infty$ on the per-system estimation error. Note that in case all systems are stable, according to \pref{thm:cov_conc}, the maximum condition number $\boldsymbol \kappa_\infty$ does \emph{not} grow with $T$, but it scales exponentially with $l^*_m$. The latter again indicates an important consequence of the largest block-sizes in Jordan decomposition that this work introduces. 

Moreover, when a transition matrix $A_m$ has eigenvalues close to or on the unit circle in the complex plane, by \pref{thm:cov_conc}, the factor $\boldsymbol \kappa_\infty$ grows polynomially with $T$. Thus, for systems with infinite memories or accumulative behaviors, misspecifications can significantly deteriorate the benefits of joint learning. Intuitively, the reason is that effects of notably small misspecifications can accumulate over time and contaminate the whole data of state trajectories, because of the unit eigenvalues of the transition matrices $A_m$. Therefore, the above strong sensitivity to deviations from the shared model for systems with unit eigenvalues seems to be unavoidable.

\begin{figure}[ht]
\centering     %%% not \center
\subfigure[System matrices $A_m$ are stable]{\label{fig:err-a}\includegraphics[width=0.8\columnwidth]{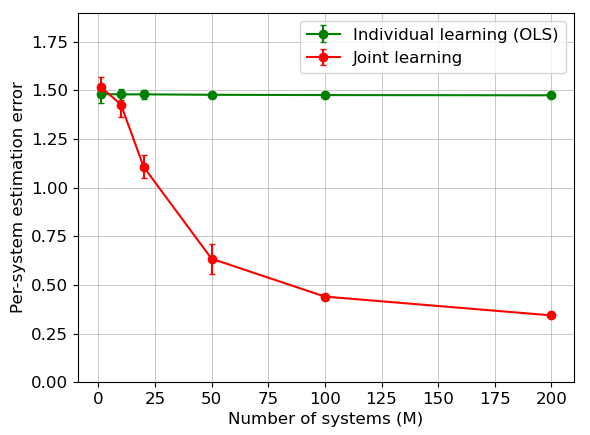}}
\subfigure[System matrices $A_m$ have unit roots]{\label{fig:err-b}\includegraphics[width=0.8\columnwidth]{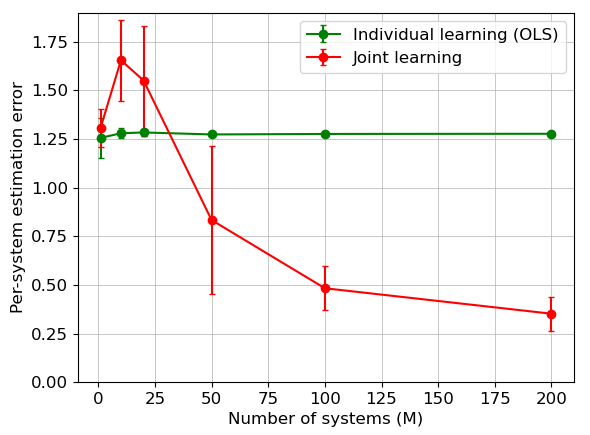}}
\caption{Per-system estimation errors vs. the number of systems $M$, for the proposed joint learning method and individual least-squares estimates of the linear dynamical systems.}\label{fig:joint-err}
\end{figure}

For example, if for the total misspecification we have $\bar \zeta^2 = O(M^{1-a})$, for some $a>0$, joint estimation improves over the individual estimators, as long as $T {\boldsymbol \kappa_\infty} \lesssim M^a{d^2}$. Hence, when all systems are stable, the joint estimation error rate improves when the number of systems satisfies $T^{1/a} \lesssim M$. Otherwise, idiosyncrasies in system dynamics dominate the commonalities. Note that larger values of $a$ correspond to \emph{smaller} misspecifications. On the other hand, \pref{thm:estimation_error_pert} implies that in systems with (almost) unit eigenvalues, the impact of $\bar \zeta^2$ is amplified. Indeed, by \pref{thm:cov_conc}, for unit-root systems, joint learning improves over individual estimators when $d^2M^a \gg T^{2l^*_m+2}$. That is, for benefiting from the shared structure and utilizing pooled data, the number of systems $M$ needs to be as large as $T^{(2l^*_m+2)/a}/d^{2/a}$. 

In contrast, if $\bar \zeta^2 = O(M^{1-a})$ for some $a>0$, the joint estimation error for the regression problem in \pref{eq:stoch_reg} incurs only an additive factor of $O(1/M^a)$, regardless of the largest block-sizes in the Jordan decompositions and unit-root eigenvalues. Thus, \pref{thm:estimation_error_pert} further highlights the stark difference between joint learning from {independent, bounded, and stationary} observations, and from state trajectories of LTI systems.

\iffalse
\section{INCORPORATING THE PRESENCE OF CONTROL INPUTS}
\aditya{Certainty equivalence based regret minimization methods for lqr control use a perturbed stabilizing controller to collect data: $u(t) = K_0 x(t) + z(t)$.}
\begin{itemize}
    \item Show a covariance upper bound and lower bound by assuming that the closed loop system $L_m = A_m + B K_m$ is stable and use the results from previous works on LTI systems.
\end{itemize}
\fi

\section{Numerical Illustrations} \label{sec:numerical}
We complement our theoretical analyses with a set of numerical experiments which demonstrate the benefit of jointly learning the systems. {We investigate two main aspects of our theoretical results: (i) benefits of joint learning when the $M$ systems share a common linear basis, for different values of $M$, and (ii) interplay of the spectral radii of the system matrices with the joint-estimation error.} To that end, we compare the estimation error for the joint estimator in \pref{eq:mt_opt} against the ordinary least-squares (OLS) estimates of the transition matrices for each system individually. For solving \pref{eq:mt_opt}, we use a minibatch gradient-descent-based implementation with Adam as the optimization algorithm \cite{kingma2015adam}. 
{Due to the bilinear form of the optimization objective, gradient descent methods can lead to convergence and computational issues for $\hat \Wb$ and $\hat B$.} Although prior studies utilize norm regularizations to address this issue in some cases \cite{tripuraneni2020provable}, we do not use any such regularization in our objective function in \eqref{eq:mt_opt}. {Notably, our unregularized minimization exposes no convergence issue in the simulations we performed.}
% \resp{Due to the bilinear form of the optimization objective, gradient descent methods can lead to unstable relative scaling of parameters $\Wb$ and $B$. Although prior studies introduced norm regularization terms to address this in the univariate regression case \cite{tripuraneni2020provable}, we did not use any such regularization in our objective in \eqref{eq:mt_opt}. Notably, our unregularized objective showed no convergence issues in simulations.}
% \resp{Given the bilinear form of the optimization objective, gradient descent methods can suffer from instability in the relative scaling of parameters $\Wb$ and $B$. While additional norm regularization terms have been used in previous works on joint-learning in the \emph{i.i.d.} univariate regression setting \cite{tripuraneni2020provable}, we did not use any such additional regularization in our objective in \eqref{eq:mt_opt}. Despite that, we did not encounter any convergence issues for the unregularized objective.} %Note that since the rank constraints are explicitly enforced in \pref{eq:mt_opt}, no projection is required during the optimization.

\begin{figure}[ht]
\centering     %%% not \center
\subfigure[System matrices $A_m$ are stable]{\label{fig:ms-err-a}\includegraphics[width=0.78\columnwidth]{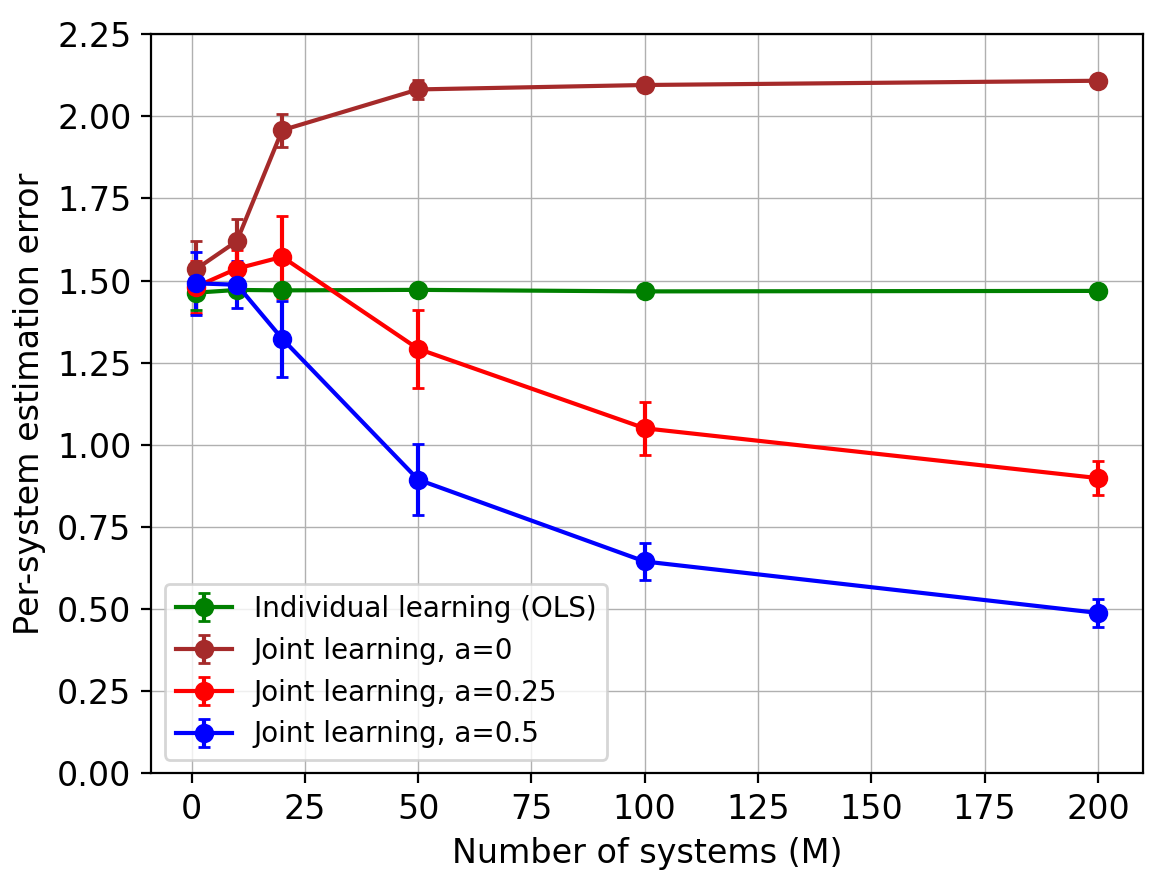}}
\subfigure[System matrices $A_m$ have a unit root]{\label{fig:ms-err-b}\includegraphics[width=0.78\columnwidth]{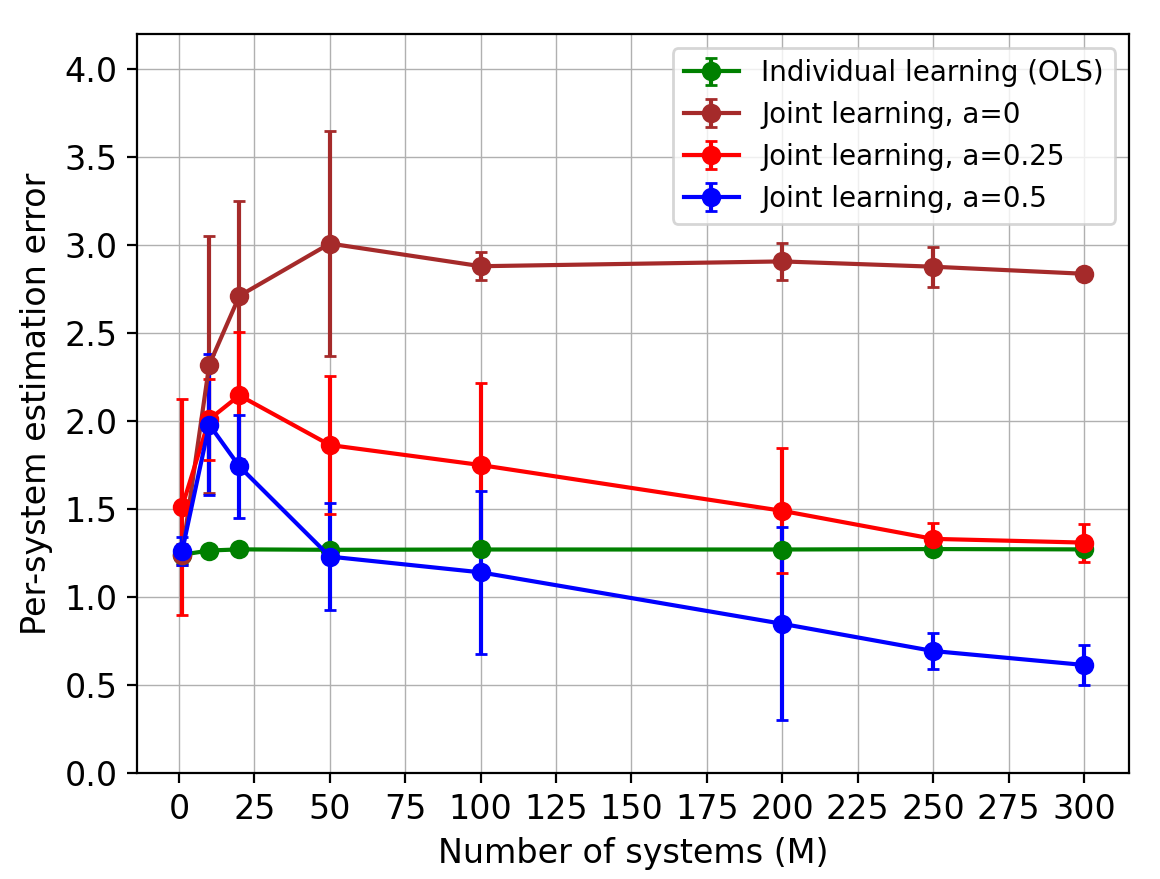}}
\caption{Per-system estimation errors are reported vs. the number of systems $M$, for varying proportions of misspecified systems; $M^{-a}$, for $a \in \{0,0.25,0.5\}$.
}\label{fig:joint-err-ms}
\end{figure}

For {generating the systems}, we consider settings with the number of bases $k=10$, dimension $d=25$, trajectory length $T=200$, and the number of systems $M \in \{1,10,20,50,100,200\}$. We simulate two cases: \\(i) the spectral radii are in the range $[0.7,0.9]$, and \\(ii) all systems have an eigenvalue of magnitude $1$. 

The matrices $\{W_i\}_{i=1}^{10}$ are generated randomly, such that each entry of $W_i$ is sampled independently from the standard normal distribution $N(0,1)$. Using these matrices, we generate $M$ systems by randomly generating the idiosyncratic components $\beta_m$ from a standard normal distribution. For generating the state trajectories, noise vectors are isotropic Gaussian with variance $4$. Additional numerical simulations using Bernoulli random matrices are provided in \pref{app:add_exp}. 

We simulate the joint learning problem both with and without model misspecifications. For the latter, deviations from the shared structure are simulated by the components $D_m$, which are added randomly with probability $1/M^a$ for $a \in \{0, 0.25, 0.5\}$. The matrices $D_m$ are generated with independent Gaussian entries of variance $0.01$, leading to
$\norm{D_m}_F^2 \approx 6.25$ and $\bar{\zeta}^2 \approx 6.25~M^{1-a}$, according to the dimension $d=25$.

To report the {results}, for each value of $M$ in \pref{fig:joint-err} (resp. \pref{fig:joint-err-ms}), we average the errors from $10$ (resp. $20$) random replicates and plot the standard deviation as the error bar. \pref{fig:joint-err} depicts the estimation errors for both stable and unit-root transition matrices, versus $M$. It can be seen that the joint estimator exhibits the expected improvement against the individual one.

More interestingly, in \pref{fig:ms-err-a}, we observe that for stable systems, the joint estimator performs worse than the individual one, when significant violations from the shared structure occurs in all systems (i.e., $a=0$). Note that it corroborates \pref{thm:estimation_error_pert}, since in this case the total misspecification $\bar \zeta^2$ \textit{scales linearly} with $M$. However, if the proportion of systems which violate the shared structure in \pref{assum:linear_model} decreases, the joint estimation error improves as expected ($a=0.25, 0.5$). 

\pref{fig:ms-err-b} depicts the estimation error for the joint estimator under misspecification for systems that have an eigenvalue on the unit circle in the complex plane. Our theoretical results suggest that the number of systems needs to be significantly larger in this case to circumvent the cost of misspecification in joint learning. The figure corroborates this result, wherein we observe that the joint estimation error is larger than the individual one, if all systems are misspecified (i.e., $a=0$). Decreases in the total misspecification (i.e., $a=0.25,0.5$) improves the error rate for joint learning, but requires larger number of systems than the stable case.

\begin{figure}[ht]
    \centering
    \includegraphics[width=0.75\columnwidth]{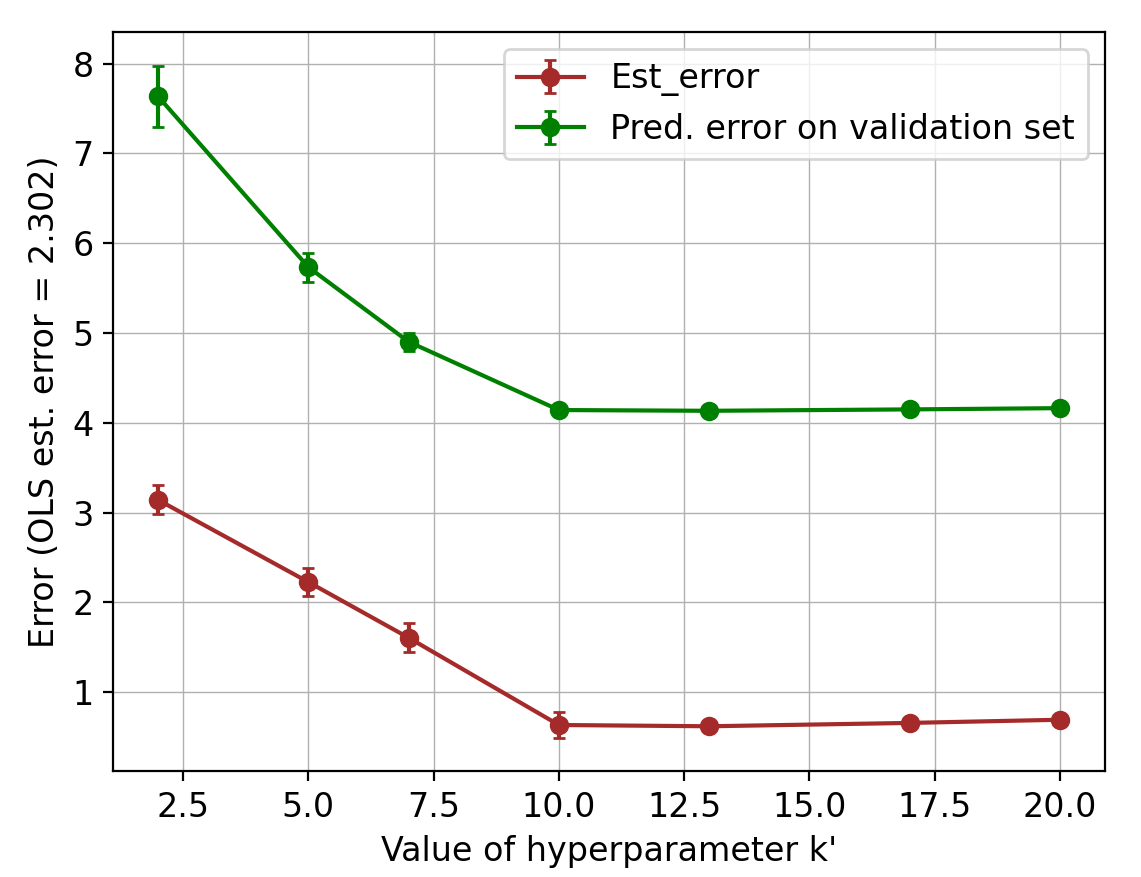}
    \caption{Estimation and validation prediction errors versus the hyperparameter $k'$, for the true value $k=10$.}
  \label{fig:choose_k}
\end{figure}

Finally, we discuss the choice of the number of bases $k$ for applying the joint estimator to real data. {It can be handled by model selection methods such as elbow criterion and information criteria \cite{akaike1974new,schwarz1978estimating}, as well as robust estimation methods in panel data and factor models~\cite{chudik2013debt,ciccone2018factor}}. In fact, for all $k' \ge k$, the structural assumption is satisfied and leads to similar learning rates, while $k' < k$ can lead to larger estimation errors. In \pref{fig:choose_k}, we provide a simulation (with $T=250, M=50$) and report the per-system estimation error, as well as the prediction error on a validation data (which is a subset of size $50$). Across all $10$ runs in the experiment, we observed that if the hyperparameter $k'$ is chosen according to the elbow criteria, the resulting number of basis models is either equal to the true value $k=10$, or slightly larger. {For misspecified models, the optimal choice of $k'$ can vary, in the sense that large misspecifications can be added to the shared basis (i.e., $k' > k$).}

% \section{Proofs for Main Results}
% \label{sec:proof}
% In this section, we provide a proof sketch of the main results. The detailed proofs can be found in the preprint available at \url{https://arxiv.org/abs/2112.10955}.
% \kazem{this way of refering is too informal. We need to cite the arxiv version as a reference. Also, each time referring to that, it must be cited.}

\section{Proof of \pref{thm:cov_conc}}
\label{sec:cov_conc_proof}
In this and the following sections, we provide the detailed proofs for our results. We start by analysing the sample covariance matrix for each system which is then used to derive the estimation error rates in \pref{thm:estimation_error} and \pref{thm:estimation_error_pert}. For clarity, some details of the proofs are delegated to \pref{app:estimation-error}. {In \pref{sec:aux_results}, we provide the general probabilistic inequalities that are used throughout the proofs.} %footnote{We have also attached an extended manuscript which includes the detailed proofs as supplementary material.}. 
Now, we prove high probability bounds for covariance matrices $\Sigma_m = \Sigma_m(T) = \sum_{t=0}^T x_m(t) x_m(t)'$ in \pref{thm:cov_conc}.

\subsection{Upper Bounds on Covariance Matrices}
To prove an upper bound on each system covariance matrix, we use an approach for LTI systems that relies on bounding norms of exponents of matrices \cite{faradonbeh2018finite}. Using $l_m^*$ and $\alpha(A_m)$ in \eqref{eq:alpha_def} and $\xi_m = \opnorm{P_m^{-1}}{\infty\rightarrow 2}\opnorm{P_m}{\infty}$, the first step is to bound the sizes of all state vectors under the event $\Ecal_{\mathrm{bdd}}(\delta)$ in \pref{prop:truncation}.
\begin{proposition}[Bounding $\norm{x_m(t)}$]
For all $t \in [T], m \in [M]$, under the event $\Ecal_{\mathrm{bdd}}(\delta)$, we have:
\begin{align*}
    % \label{eq:state_size}
    \norm{x_m(t)} \le \begin{cases} 
    \alpha(A_m) \bar b_m(\delta), & \text{if }~~ |\lambda_{m,1}| < 1-\frac{\rho}{T}, \\
    \alpha(A_m) \bar b_m(\delta) t^{l^*_{m}}, & \text{if }~~ |\lambda_{m,1}-1| \le \frac{\rho}{T}
    \end{cases}.
\end{align*}
where $\bar b_m(\delta) = \rbr{b_T(\delta) + \norm{x_m(0)}_\infty}$.
\end{proposition}
\begin{proof}
As before, each transition matrix $A_m$ admits a Jordan normal form as follows: $A_m = P_m^{-1}\Lambda_m P_m$, where $\Lambda_m$ is a block-diagonal matrix $\Lambda_m = \diag\rbr{\Lambda_{m,q},\ldots,\Lambda_{m,q}}$. Each Jordan block $\Lambda_{m,i}$ is of size $l_{m,i}$. Note that for each system, the state vector satisfies:
\begin{align*}
    x_m(t) = {} & \sum_{s=1}^t A_m^{t-s} \eta_m(s) + A_m^t x_m(0) \\
    = {} & \sum_{s=1}^t P_m^{-1}\Lambda_m^{t-s}P_m \eta_m(s) + P_m^{-1}\Lambda_m^t P x_m(0).
\end{align*}
Now, letting $b_T(\delta)$ be the same as in \pref{prop:truncation}, we can bound the $\ell_2$-norm of the state vector as follows:
\begin{align*}
    \norm{x_m(t)} \le {} & \opnorm{P_m^{-1}}{\infty \rightarrow 2} \opnorm{\sum_{s=1}^t \Lambda_m^{t-s}}{\infty} \opnorm{P_m}{\infty}b_T(\lambda) \\ {} & + \opnorm{P_m^{-1}}{\infty \rightarrow 2} \opnorm{\Lambda_m^{t}}{\infty} \opnorm{P_m}{\infty}\norm{x_m(0)}_\infty \\
    \le {} & \xi_m\rbr{\sum_{s=0}^t \opnorm{\Lambda_m^{t-s}}{\infty}} \big(b_T(\delta) + \norm{x_m(0)}_\infty\big).
\end{align*}
For any matrix, the $\ell_\infty$ norm is equal to the maximum row sum. Since the powers of a Jordan matrix will follow the same block structure as the original one, we can bound the operator norm $\opnorm{A_m^{t-s}}{\infty}$ by the norm of each block. 
% For any Jordan matrix of size $l$ and eigenvalue $\lambda$, we have: 
% \begin{align*}
%     \Lambda^s = \begin{bmatrix}
%     \lambda^s & \binom{s}{1}\lambda^{s-1} & \ldots & \binom{s}{l-1}\lambda^{s-l+1}\\
%     0 & \lambda^s & \ldots & \binom{s}{l-2}\lambda^{s-l+2}\\
%     \vdots & \vdots & \ddots & \vdots\\
%     0 & 0 & \ldots & \lambda^s
%     \end{bmatrix},
% \end{align*}
The maximum row sum for the $s$-th power of a Jordan block is: $\sum_{j=0}^{l-1} \binom{s}{j} \lambda^{s-j}$. Using this, we will bound the size of each state vector for the case when 
\begin{enumerate}[label=(\Roman*)]
    \item the spectral radius of $A_m$ satisfies $\abr{\lambda_1(A_m)} < 1-\frac{\rho}{T}$,
    \item or, when $\abr{\lambda_1(A_m)-1} \le \frac{\rho}{T}$, for a constant $\rho>0$.
\end{enumerate}

\paragraph*{Case I} When the Jordan block for a system matrix has eigenvalues strictly less than 1, we have:
\begin{align*}
    \sum_{s=0}^t \opnorm{\Lambda_m^{t-s}}{\infty} \le {} & \max_{i \in [q_m]} \sum_{s=0}^{t} \sum_{j=0}^{l_{m,i}-1} \binom{s}{j} \abr{\lambda_{m,i}}^{s-j} \\
    \le {} & \sum_{s=0}^{t} \sum_{j=0}^{l^*_m-1} \binom{s}{j} \abr{\lambda_{m,1}}^{s-j} \\
    \le {} & \sum_{s=0}^{t} \abr{\lambda_{m,1}}^s \sum_{j=0}^{l^*_{m}-1} \frac{s^j}{j!} \abr{\lambda_{m,1}}^{-j} \\
    \le {} & \sum_{s=0}^{t} \abr{\lambda_{m,1}}^s s^{l^*_{m}-1} \sum_{j=0}^{l^*_{m}-1} \frac{\abr{\lambda_{m,1}}^{-j}}{j!}   \\
    \le {} & e^{1/|\lambda_{m,1}|} \sum_{s=0}^{\infty} \abr{\lambda_{m,1}}^s s^{l^*_{m}-1} \\
    % \le {} & e^{1/|\lambda_{m,1}|} \sum_{s=0}^{\infty} \abr{\lambda_{m,1}}^s s^{l^*_{m}-1}  \\
    \lesssim {} & e^{1/|\lambda_{m,1}|} \sbr{\frac{l^*_{m}-1}{- \log |\lambda_{m,1}|} + \frac{(l^*_{m}-1)!}{(-\log |\lambda_{m,1}|)^{l^*_{m}}}}.
\end{align*}
Thus, for this case, each state vector can be upper bounded as $\norm{x_m(t)} \le \alpha(A_m) (b_T(\delta) + \norm{x_m(0)}_\infty)$.
When the matrix $A_m$ is diagonalizable, each Jordan block is of size $1$, which leads to the upper-bound $\sum_{s=0}^t \opnorm{\Lambda_m^{t-s}}{\infty} \le ( 1-\lambda_1 )^{-1}$, for all $t \ge 0$. Therefore for diagonalizable $A_m$, we can let $\alpha(A_m) = 
(1-\lambda_1)^{-1} {\opnorm{P_m^{-1}}{\infty\rightarrow 2}\opnorm{P_m}{\infty}} $.

\paragraph*{Case II} When $\abr{\lambda_{m,1}-1} \le \frac{\rho}{T}$, we get $\abr{\lambda_{m,1}}^t \le e^{\rho}$, for all $t \le T$. Therefore, since $l^*_m$ is the largest Jordan block, we have:
\begin{align*}
    \sum_{s=0}^t  \opnorm{\Lambda_m^{t-s}}{\infty} \le {} & \sum_{s=0}^{t} \sum_{j=0}^{l^*_m-1} \binom{s}{j} \abr{\lambda_{m,1}}^{s-j} \le {} e^{\rho} \sum_{s=0}^{t} \sum_{j=0}^{l^*_m-1} \binom{s}{j} \\
    \le {} & e^\rho \sum_{s=0}^{t} \sum_{j=0}^{l^*_m-1} s^j/j!
    \le {} e^\rho \sum_{s=0}^{t} s^{l^*_m-1}\sum_{j=0}^{l^*_m-1} 1/j! \\
    \le {} & e^{\rho+1} \sum_{s=0}^{t}s^{l^*_m-1} \lesssim {} e^{\rho+1} t^{l^*_m}.
\end{align*}
Therefore, the magnitude of each state vector grows polynomially with $t$, the exponent being at most $l^*_m$. For example, when $A_m$ is diagonalizable, the Jordan block for the unit root is of size $1$, givin $\sum_{s=0}^t  \opnorm{\Lambda_m^{t-s}}{\infty} \le e^\rho t $.

So, for systems with unit roots, the bound on each state vector is as expressed in the proposition.
%$\norm{x_m(t)} \le \alpha(A_m) \rbr{ b_T(\delta) + \norm{x_m(0)}_\infty} t^{l^*_m}$.
\end{proof}

Using the high probability upper bound on the size of each state vector, we can upper bound the covariance matrix for each system as follows:
\begin{lemma}[Upper bound on $\Sigma_m$]
\label{lem:cov_upper_bound}
For all $m \in [M]$, the sample covariance matrix $\Sigma_m$ of system $m$ can be upper bounded under the event $\Ecal_{\mathrm{bdd}}(\delta)$, as follows:
\begin{enumerate}[label=(\Roman*)]
    \item When all eigenvalues of the matrix $A_m$ are strictly less than $1$ in magnitude ($|\lambda_{m,i}| < 1-\frac{\rho}{T}$), we have
    \begin{align*}
        \lambda_{\max}(\Sigma_m) \le \alpha(A_m)^2 \rbr{b_T(\delta) + \norm{x_m(0)}_\infty}^2 T.
    \end{align*}
    \item When some eigenvalues of the matrix $A_m$ are close to $1$, i.e. $|\lambda_1(A_m)-1| \le \frac{\rho}{T}$, we have:
    \begin{align*}
        \lambda_{\max}(\Sigma_m) \le \alpha(A_m)^2 \rbr{b_T(\delta) + \norm{x_m(0)}_\infty}^2 T^{2l_{m,1} + 1}.
    \end{align*}
\end{enumerate}
\end{lemma}
\begin{proof}
First note that we have:
\begin{align*}
    \lambda_{\max}(\Sigma_m) = \opnorm{\sum_{t=0}^T x_m(t) x_m(t)'}{2} \le \sum_{t=0}^T \norm{x_m(t)}_2^2.
\end{align*}
Therefore, by \pref{prop:truncation}, when all eigenvalues of $A_m$ are strictly less than $1$, we have:
\begin{align*}
    \lambda_{\max}(\Sigma_m) \le T \alpha(A_m)^2 \rbr{b_T(\delta) + \norm{x_m(0)}_\infty}^2.
\end{align*}
For the case when $1-\frac{\rho}{T} \le \lambda_1(A_m) \le 1+\frac{\rho}{T}$, we get:
\begin{align*}
    \lambda_{\max}(\Sigma_m) \le {} & \alpha(\Lambda_m)^2 \sum_{t=0}^T t^{2l_{m,1}} \\
    \le {} & \alpha(A_m)^2 \rbr{b_T(\delta) + \norm{x_m(0)}_\infty}^2 T^{2l_{m,1} + 1}.
\end{align*}
\end{proof}

\subsection{Lower Bound for Covariance Matrices}
\label{sec:cov_lowbnd}
A lower bound result for the idiosyncratic covariance matrices can be derived using the probabilistic inequalities in the last section. We provide a detailed proof below.

\begin{lemma}[Covariance lower bound.] \label{lem:cov_lowbnd}
Define $\varkappa = \tfrac{d\sigma^2}{\lambda_{\min}(C)^2}$. For all $m \in [M]$, if the per-system sample size $T$ is greater than $T_0$ defined as 
% \begin{small}
\begin{align*}
    \varkappa \cdot {\max}\rbr{c_\eta \log \tfrac{18}{\delta} , 16\rbr{ \log \rbr{\alpha(A)^2\bar b_m(\delta)^2 + 1} + 2 \log \tfrac{5}{\delta}}},
\end{align*}
{\small if $\abr{\lambda_{m,1}} < 1-\frac{\rho}{T}$, and
 \begin{align*}
    \varkappa \cdot {\max}\rbr{c_\eta \log \tfrac{18}{\delta} , 16\rbr{ \log \rbr{\alpha(A)^2\bar b_m(\delta)^2T^{2l^*_m} + 1} + 2 \log \tfrac{5}{\delta}}}
\end{align*}}
if $1-\frac{\rho}{T} \le \abr{\lambda_{m,1}} \le 1+\frac{\rho}{T}$,
% \end{small}
% \kazem{the above eqs are to busy and into the margin.}
then with probability at least $1-3\delta$, the sample covariance matrix $\Sigma_m$ for system $m$ can be bounded from below: $\Sigma_m(T) \succeq \frac{T\lambda_{\min}(C)}{4}I$.
\end{lemma}
\begin{proof}
We bound the covariance matrix under the events $\Ecal_{\mathrm{bdd}}(\delta)$, $\Ecal_{\eta}(\delta)$ in Propositions \ref{prop:truncation}, \ref{prop:noise-conc}, as well as the one in \pref{prop:matrix-self-norm}. As we consider a bound for all systems, we drop the system subscript $m$ here. Using \pref{eq:mt-lti}, we have:
\begin{align*}
    \Sigma(T) \succeq {} & A\Sigma(T-1)A' + \sum_{t=1}^T \eta(t)\eta(t)' \\
    {} & + \sum_{t=0}^{T-1} \rbr{Ax(t)\eta(t+1)' + \eta(t+1)x(t)'A'}
\end{align*}
Since $T \ge \varkappa c_\eta \log \frac{18}{\delta} = \frac{c_\eta d\sigma^2 \log \frac{18}{\delta}}{\lambda_{\min}(C)^2}$, under the event $\Ecal_{\eta}(\delta)$ it holds that
\begin{align*}
    \Sigma(T) \succeq {} & A\Sigma(T-1)A' + \frac{3\lambda_{\min}(C)T}{4}\\
    {} & + \sum_{t=0}^{T-1} \rbr{Ax(t)\eta(t+1)' + \eta(t+1)x(t)'A'}.
\end{align*}
Thus, for any unit vector $u$ (i.e., on the unit sphere $\Scal^{d-1}$), we have
\begin{align*}
    u'\Sigma(T)u \ge {} & u'A\Sigma(T-1)A'u + \frac{3\lambda_{\min}(C)T}{4}\\
    {} & + \sum_{t=0}^{T-1} u'\rbr{Ax(t)\eta(t+1)' + \eta(t+1)x(t)'A'}u.
\end{align*}
Now, by \pref{prop:matrix-self-norm} with $V = T \cdot I$, we get the following result for the martingale $\sum_{t=0}^{T-1} A_m X_m(t)\eta_m(t+1)'$ and $\bar V_m(s) \coloneqq \sum_{t=0}^s A_m X_m(t)X_m(t)'A_m' + V$, with probability at least $1-\delta$:
\begin{align*}
    & \norm{\sum_{t=0}^{T-1} Ax(t)\eta(t+1)'u} & \\
    \le {} & \sqrt{u'A\Sigma(T-1)A'u + T} \\
    {} & \sqrt{8d\sigma^2 \log \rbr{\frac{5\det \rbr{\bar{V}_m(T-1)}^{1/2d} \det \rbr{TI}^{-1/2d} }{\delta^{1/d}}}}.
\end{align*}
Thus, we get:
\begin{align*}
    & u'\Sigma(T)u & \\
    \succeq {} & u'A\Sigma(T-1)A'u - \sqrt{u'A\Sigma(T-1)A'u + T}\\
    {} & \sqrt{16d\sigma^2 \log \rbr{\tfrac{\lambda_{\max}(\bar V(T-1))}{T}} + 32d\sigma^2 \log \tfrac{5}{\delta} } + \frac{3\lambda_{\min}(C)T}{4}. 
\end{align*}
Hence, we have:
\begin{align*}
    u'\frac{\Sigma(T)}{T}u \succeq {} & u'\frac{A\Sigma(T-1)A'}{T}u  + \frac{3\lambda_{\min}(C)}{4} \\
    {} & - \sqrt{u'\frac{A\Sigma(T-1)A'}{T}u + 1} \frac{\lambda_{\min}(C)}{2} \\
    \succeq {} & \frac{\lambda_{\min}(C)}{4},
\end{align*}
whenever $T$ is larger than
\begin{align*}
    & \frac{16d\sigma^2}{\lambda_{\min}(C)^2}\rbr{ \log \rbr{\frac{\lambda_{\max}(\bar V(T-1))}{T}} + 2 \log \frac{5}{\delta}} & \\
    = {} & \tfrac{16d\sigma^2}{\lambda_{\min}(C)^2}\rbr{ \log \rbr{\tfrac{\lambda_{\max}\rbr{\sum_{t=0}^{T-1} A X(t)X(t)'A'}}{T} + 1} + 2 \log \tfrac{5}{\delta}}. &
\end{align*}
Using the upper bound analysis in \pref{lem:cov_upper_bound}, we show that it suffices for $T$ to be lower bounded as
\begin{align*}
    T \ge \frac{16d\sigma^2}{\lambda_{\min}(C)^2}\rbr{ \log \rbr{\alpha(A)^2\bar b_m(\delta)^2 + 1} + 2 \log \frac{5}{\delta}},
\end{align*}
when $A$ is strictly stable, and as
\begin{align*}
    T \ge \frac{16d\sigma^2}{\lambda_{\min}(C)^2}\rbr{ \log \rbr{\alpha(A)^2 \bar b_m(\delta)^2 T^{2l^*} + 1} + 2 \log \frac{5}{\delta}},
\end{align*}
when $\abr{\lambda_1(A)} \le 1+\frac{\rho}{T}$. Since, both quantities on the RHS grow at most logarithmically with $T$, there exists $T_0$ such that it holds for all $T \ge T_0$. Combining the failure probability for all events, we get the desired result.
\end{proof}

\section{Proof of \pref{thm:estimation_error}}
\label{sec:estimation-error-proof}
In this section, we use the result in \pref{thm:cov_conc} to analyze the estimation error for the estimator in \pref{eq:mt_opt}, under \pref{assum:linear_model}. For ease of presentation, we rewrite the problem by transforming the vector output space to scalar values. For that purpose, we introduce some notation to express transition matrices in vector form and rewrite \pref{eq:mt_opt}. First, for each state vector $x_m(t) \in \R^d$, we create $d$ different covariates of size $\R^{d^2}$. So, for $j=1,\cdots, d$, the vector $\tilde{x}_{m,j}(t) \in \R^{d^2}$ contains $x_m(t)$ in the $j$-th block of size $d$ and $0's$ elsewhere.
    
Then, we express the system matrix $A_m \in \R^{d \times d}$ as a vector $\tilde{A}_m \in \R^{d^2}$. Similarly, the concatenation of all vectors $\tilde{A}_m$ can be coalesced into the matrix $\tilde{\Theta} \in \R^{d^2 \times M}$. Analogously, $\tilde{\eta}_m(t)$ will denote the concatenated $dt$ dimensional vector of noise vectors for system $m$. Thus, the structural assumption in \pref{eq:linear_model} can be written as:
\begin{align}
    \label{eq:low_rank}
    \tilde{A}_m = W^*\beta^*_m,
\end{align}
where $W^* \in \R^{d^2 \times k}$ and $\beta^*_m \in \R^k$. Similarly, the overall parameter set can be factorized as $\tilde{\Theta}^* = W^*B^*$, where the matrix $B^* = [\beta^*_1 | \beta^*_2 | \cdots \beta^*_M] \in \R^{k \times M}$ contains the true weight vectors $\beta^*_m$. Thus, expressing the system matrices $A_m$ in this manner leads to a low rank structure in \pref{eq:low_rank}, so that the matrix $\tilde \Theta^*$ is of rank $k$. Using the vectorized parameters, the evolution for the components $j \in [d]$ of all state vectors $x_m(t)$ can be written as:
\begin{align}
    x_{m}(t+1)[j] = \tilde{A}_m\tilde{x}_{m,j}(t) + \eta_{m}(t+1)[j].
\end{align}
For each system $m \in [M]$, we therefore have a total of $dT$ samples, where the statistical dependence now follows a block structure: $d$ covariates of $x_m(1)$ are all constructed using $x_m(0)$, next $d$ using $x_m(1)$ and so forth. To estimate the parameters, we solve the following optimization problem:
\begin{align}
\label{eq:mt_uopt}
    & \hat{W}, \{\hat{\beta}_m\}_{m=1}^M & \nonumber \\
    \coloneqq & {} \argmin_{W, \{\beta_m\}_{m=1}^M} \underbrace{ \sum_{m,t}\sum_{j=1}^d \rbr{ x_{m}(t+1)[j] - \langle W \beta_m,  \tilde{x}_{m,j}(t)\rangle }^2 }_{\Lcal(W, \beta)} & \nonumber \\
    = & {} \argmin_{W, \{\beta_m\}_{m=1}^M} \sum_{m=1}^M \nbr{y_m - \tilde{X}_m W \beta_m }_2^2, & 
\end{align}
where $y_m \in \R^{Td}$ contains all $T$ state vectors stacked vertically and $\tilde{X}_m \in \R^{Td \times d^2}$ contains the corresponding matrix input.  We denote the covariance matrices for the vectorized form by $\tilde{\Sigma}_m=\sum_{t=0}^{T-1} \tilde x_m(t) \tilde x_m(t)'$. Recall, that the sample covariance matrices for all systems are denoted by $\Sigma_m=\sum_{t=0}^{T-1} x_m(t)x_m(t)'$

We further use the following notation: for any parameter set $\Theta = WB \in \R^{d^2 \times M}$, we define $\Xcal(\Theta) \in \R^{dT \times M}$ as
$    \Xcal(\Theta) \coloneqq \sbr{\Xcal_1(\Theta) | \Xcal_2(\Theta) \cdots | \Xcal_M(\Theta)}$, where each column $\Xcal_m(\Theta) \in \R^{dT}$ is the prediction of states $x_m(t+1)$ with $\Theta_m$. That is, \[\Xcal_m(\Theta) = (x_m(0)', x_m(0)'\Theta_m', \ldots , x_m(T-1)'\Theta_m' )'.\] 
Thus, $\Xcal(\tilde \Theta^*) \in \R^{Td \times M}$ denotes the ground truth mapping for the training data of the $M$ systems and $\Xcal(\tilde \Theta^* - \hat \Theta) \in \R^{Td \times M}$ is the prediction error across all coordinates of the $MT$ state vectors that each of which is of dimension $d$. 

By \pref{assum:linear_model}, we have $\Delta \coloneqq \tilde \Theta^* - \hat{\Theta} = UR$, where $U \in O^{d^2 \times 2k}$ is an orthonormal matrix and $R \in \R^{2k \times M}$. We start by the fact that the estimates $\hat{W}$ and $\hat{\beta}_m$ minimize \pref{eq:mt_opt}, and therefore, have a smaller squared prediction error than $(W^*, B^*)$. Hence, we get the following inequality:
\begin{align}
\label{eq:sqloss_ineq}
    & \frac{1}{2} \sum_{m=1}^M \norm{ \tilde{X}_m (W^*\beta^*_m - \hat{W} \hat{\beta}_m) }_2^2 & \nonumber \\
    \le & {} \sum_{m=1}^M \inner{\tilde \eta_m}{\tilde{X}_m \rbr{\hat{W}\hat{\beta}_m - W^*\beta^*_m}} . &
\end{align}
We can rewrite $\hat{W}\hat{\beta}_m - W^*\beta^*_m=Ur_m$, for all $m \in [M]$, where $r_m \in \R^{2k}$ is an idiosyncratic projection vector for system $m$. {Since our joint estimator is a least squares objective with bilinear terms, we first decompose the prediction error for the estimator, similar to the linear regression setting \cite{du2020few,tripuraneni2020provable}. In subsequent analyses, we use different matrix concentration results and LTI estimation theory in order to account for the temporal dependence and spectral properties of the systems.} Our first step is to bound the prediction error for all systems.
\begin{lemma}
\label{lem:breakup}
For any fixed orthonormal matrix $\bar{U} \in \R^{d^2 \times 2k}$, the total squared prediction error in \pref{eq:mt_opt} for $(\hat W, \hat B)$ can be decomposed as follows:
\begin{align}
    & \frac{1}{2}\sum_{m=1}^M \norm{\tilde{X}_m (W^*\beta^*_m - \hat{W} \hat{\beta}_m)}_F^2 & \nonumber\\ 
    \le {} & \sqrt{\sum_{m=1}^M \norm{\tilde\eta_m^\top \tilde{X}_m \bar{U}}_{\bar{V}^{-1}_m}^2}\sqrt{2\sum_{m=1}^M \norm{ \tilde{X}_m (W^*\beta^*_m - \hat{W} \hat{\beta}_m) }_2^2} \nonumber \\
    {} & + \sqrt{\sum_{m=1}^M \norm{\tilde\eta_m^\top \tilde{X}_m \bar{U}}_{\bar{V}^{-1}_m}^2} \sqrt{\sum_{m=1}^M \norm{\tilde{X}_m(\bar{U}-U)r_m}^2} \nonumber \\
    {} & + \sum_{m=1}^M \inner{\tilde\eta_m}{\tilde{X}_m (U-\bar{U}) r_m} . \label{eq:pred_error_breakup} & 
\end{align}
\end{lemma}
{The proof of \pref{lem:breakup} can be found in the extended version of this paper \cite{modi2021joint}. Our next step is to bound each term on the RHS of \eqref{eq:pred_error_breakup}. To that end, let $\Ncal_{\epsilon}$ be an $\epsilon$-cover of the set of orthonormal matrices in $\R^{d^2 \times 2k}$. In \eqref{eq:pred_error_breakup}, we select the matrix $\bar{U}$ to be an element of $\Ncal_{\epsilon}$ such that $\norm{\bar{U} - U}_F \le \epsilon$. Note that since $\Ncal_{\epsilon}$ is an $\epsilon$-cover, such matrix $\bar{U}$ exists. We can bound the size of such a cover using \pref{lem:lowrank-cover}, and obtain $|\Ncal_\epsilon| \le \rbr{\frac{6\sqrt{d}}{\epsilon}}^{2d^2k}$.

We now bound each term in the following propositions using the auxiliary results in \pref{sec:aux_results} and covariance matrix bounds in the previous section. The detailed proofs for all of the following results are available in the extended version \cite{modi2021joint}. Using \pref{prop:noise-magnitude}, we bound the following expression in the second term of \eqref{eq:pred_error_breakup}, as follows.
\begin{proposition}
\label{prop:proj_cover_error}
Under \pref{assum:linear_model}, for the noise process $\{\eta_m(t)\}_{t=1}^\infty$ defined for each system, with probability at least $1-\delta_Z$, we have:
\begin{align*}
    \sum_{m=1}^M \norm{\tilde{X}_m(\bar{U}-U)r_m}^2 \lesssim {} & \boldsymbol\kappa\epsilon^2 \rbr{MT\tr{C} +\sigma^2\log \frac{2}{\delta_Z} }. %\label{eq:pred_error_coarse}
\end{align*}
\end{proposition}

Based on the bound in \pref{prop:proj_cover_error}, we can bound the third term in \eqref{eq:pred_error_breakup} as follows:
\begin{proposition}
\label{prop:breakup_inner_prod}
Under \pref{assum:noise} and \pref{assum:linear_model}, with probability at least $1-\delta_Z$, we have:
\begin{align}
\label{eq:breakup_inner_prod}
    \sum_{m=1}^M \inner{\tilde\eta_m}{\tilde{X}_m (U-\bar{U}) r_m}
    \lesssim {} & \sqrt{\boldsymbol\kappa} \epsilon \rbr{MT \tr{C} + \sigma^2\log \tfrac{1}{\delta_Z}}.
\end{align}
\end{proposition}

Next, we show a multitask concentration of martingales projected on a low-rank subspace.
\begin{proposition}
\label{prop:mt-self-norm}
For an arbitrary orthonormal matrix $\bar{U} \in \R^{d^2 \times 2k}$ in the $\epsilon$-cover $\Ncal_\epsilon$ defined in \pref{lem:lowrank-cover}, let $\Sigma \in \R^{d^2 \times d^2}$ be a positive definite matrix, and define $S_{m}(\tau) = \tilde \eta_m(\tau)^\top \tilde{X}_m(\tau) \bar{U}$, $\bar{V}_m(\tau) = \bar{U}'\rbr{\tilde{\Sigma}_m(\tau) + \Sigma}\bar{U}$, and $V_0 = \bar{U}'\Sigma\bar{U}$. Then, letting $\Ecal_{1}(\delta_U)$ be the event 
\begin{align*}
    \sum_{m=1}^M \norm{S_{m}(T)}_{\bar{V}_m^{-1}(T)}^2 \le 2\sigma^2\log \rbr{\frac{\Pi_{m=1}^M\frac{\det(\bar{V}_m(T))}{\det (V_0)}}{\delta_U}},
\end{align*}
we have
\begin{align}
\label{eq:mt-self-norm}
    \PP\sbr{\Ecal_1(\delta_U)} \ge 1-\rbr{\frac{6\sqrt{2k}}{\epsilon}}^{2d^2k}\delta_U.
\end{align}
\end{proposition}

\subsection{Proof of Estimation Error in \pref{thm:estimation_error}}
\label{sec:final_proof}
\begin{proof}
We now use the bounds we have shown for each term before and give the final steps by using the error decomposition in \pref{lem:breakup}.
% From \pref{lem:breakup}, we have: 
% \begin{align*}
%     & \norm{\Xcal (W^*B^* - \hat{W} \hat{B})}_F^2 & \\
%     \le {} & \sqrt{2\sum_{m=1}^M \norm{\tilde\eta_m^\top \tilde{X}_m \bar{U}}_{\bar{V}^{-1}_m}^2}\norm{\Xcal (W^*B^* - \hat{W} \hat{B})}_F &\\
%     {} & + \sqrt{\sum_{m=1}^M \norm{\tilde\eta_m^\top \tilde{X}_m \bar{U}}_{\bar{V}^{-1}_m}^2} \sqrt{2\sum_{m=1}^M \norm{\tilde{X}_m(\bar{U}-U)r_m}^2} & \\
%     {} & + \sum_{m=1}^M \inner{\tilde\eta_m}{\tilde{X}_m (U-\bar{U}) r_m}. &
% \end{align*}
Let $\abr{\Ncal_{\epsilon}}$ be the cardinality of the $\epsilon$-cover of the set of orthonormal matrices in $\R^{d^2 \times 2k}$ that we defined in \pref{lem:breakup}. Let $\VV$ denote the expression $\Pi_{m=1}^M\frac{\det(\bar{V}_m(t))}{\det (V_0)}$. So, substituting the termwise bounds from \pref{prop:proj_cover_error}, \pref{prop:breakup_inner_prod}, and \pref{prop:mt-self-norm} in \pref{lem:breakup}, with probability at least $1-\abr{\Ncal_{\epsilon}}\delta_U - \delta_Z$, it holds that:
\begin{align}
    & \frac{1}{2}\norm{\Xcal (W^*B^* - \hat{W} \hat{B})}_F^2 & \nonumber \\
    \lesssim {} & \sqrt{\sigma^2\log \rbr{\frac{\VV}{\delta_U}}}\norm{\Xcal (W^*B^* - \hat{W} \hat{B})}_F \nonumber & \\
    {} & + \sqrt{\sigma^2\log \rbr{\frac{\VV}{\delta_U}}} \sqrt{\boldsymbol\kappa\epsilon^2 \rbr{MT\tr{C} +\sigma^2\log \frac{1}{\delta_Z} }} \nonumber & \\
    {} & + \sqrt{\boldsymbol\kappa} \epsilon \rbr{MT \tr{C} + \sigma^2\log \frac{1}{\delta_Z}}. & \label{eq:inter_breakup}
\end{align}
For the matrix $V_0$, we now substitute $\Sigma = \underbar{\lambda}I_{d^2}$, which implies that $\det(V_0)^{-1} = \det(1/\underbar{\lambda}I_{2k}) = \rbr{1/\underbar{\lambda}}^{2k}$. Similarly, for the matrix $\bar{V}_m(T)$, we get $\det(\bar{V}_m(T)) \le \bar{\lambda}^{2k}$. Thus, substituting $\delta_U = 3^{-1}\delta\abr{\Ncal}_\epsilon^{-1}$ and $\delta_C=3^{-1}\delta$ in \pref{thm:cov_conc}, with probability at least $1-2\delta/3$, the upper-bound in \pref{prop:mt-self-norm} becomes:
\begin{align*}
    \sum_{m=1}^M \norm{\tilde\eta_m^\top \tilde{X}_m \bar{U}}_{\bar{V}^{-1}_m}^2 \le {} & \sigma^2\log \rbr{\frac{\Pi_{m=1}^M\frac{\det(\bar{V}_m(t))}{\det (V_0)}}{\delta_U}} \\
    % \le {} & \sigma^2 \log \rbr{ \frac{\bar{\lambda}}{\underbar{\lambda}} }^{2Mk} + \sigma^2 \log \rbr{ \frac{18k}{\delta \epsilon} }^{2d^2k} \\
    \lesssim {} & \sigma^2Mk \log \boldsymbol\kappa_\infty + \sigma^2 d^2k \log \frac{k}{\delta \epsilon}.
\end{align*}
Substituting this in \pref{eq:inter_breakup} with $\delta_Z = \delta/3$, $c^2 = \max(\sigma^2, \lambda_1(C))$, with probability at least $1-\delta$, we have:
\begin{align*}
    & \frac{1}{2} \norm{\Xcal (W^*B^* - \hat{W} \hat{B})}_F^2 & \\ 
    \lesssim {} & \sqrt{c^2Mk \log \boldsymbol\kappa_\infty + c^2 d^2k \log \frac{k}{\delta \epsilon}}\Bigg(\norm{\Xcal (W^*B^* - \hat{W} \hat{B})}_F & \\
    {} & +  \sqrt{\boldsymbol\kappa\epsilon^2 \rbr{c^2dMT +c^2\log \frac{1}{\delta}}}\Bigg) & \\
    {} & + \sqrt{\boldsymbol\kappa} \epsilon \rbr{c^2dMT + c^2\log \frac{1}{\delta}}  . & %%%%%%%%%%%%
\end{align*}

Noting that $\log \frac{1}{\delta} \lesssim d^2k\log \frac{k}{\delta \epsilon}$ for $\epsilon = \frac{k}{\sqrt{\boldsymbol\kappa}d^2T}$, with probability at least $1-\delta$, we get:
\begin{small}
\begin{align*}
    & \frac{1}{2} \norm{\Xcal (W^*B^* - \hat{W} \hat{B})}_F^2 & \\
    \lesssim {} & \rbr{\sqrt{c^2Mk \log \boldsymbol\kappa_\infty + c^2 d^2k \log \tfrac{\boldsymbol\kappa dT}{\delta}}}\norm{\Xcal (W^*B^* - \hat{W} \hat{B})}_F  \\
    {} & + \sqrt{c^2Mk \log \boldsymbol\kappa_\infty + c^2 d^2k \log \tfrac{\boldsymbol\kappa dT}{\delta}} \sqrt{c^2\rbr{\tfrac{k^2M}{d^3T} + \tfrac{k^3}{d^2T^2}\log \tfrac{\boldsymbol\kappa dT}{\delta}  }} \\
    {} & + c^2\rbr{\tfrac{Mk}{d} + \tfrac{k^2}{T}\log \tfrac{\boldsymbol\kappa dT}{\delta}}.  %%%%%%%%%%%%
    % \lesssim {} & \sqrt{c^2\rbr{Mk \log \boldsymbol\kappa_\infty + d^2k \log \frac{\boldsymbol\kappa dT}{\delta}}} \norm{\Xcal (W^*B^* - \hat{W} \hat{B})}_F & \\
    % {} & + c^2\rbr{Mk \log \boldsymbol\kappa_\infty + \frac{d^2k}{T}\log \frac{\boldsymbol\kappa dT}{\delta}}. &
\end{align*}
\end{small}
As $k \le d^2$, we can rewrite the above inequality as:
\begin{align*}
    & \frac{1}{2} \norm{\Xcal (W^*B^* - \hat{W} \hat{B})}_F^2 & \\
    \lesssim {} & \sqrt{c^2\rbr{Mk \log \boldsymbol\kappa_\infty + d^2k \log \frac{\boldsymbol\kappa dT}{\delta}}} \norm{\Xcal (W^*B^* - \hat{W} \hat{B})}_F & \\
    {} & + c^2\rbr{Mk \log \boldsymbol\kappa_\infty + \frac{d^2k}{T}\log \frac{\boldsymbol\kappa dT}{\delta}}. &
\end{align*}
The above quadratic inequality for the prediction error $\norm{\Xcal(W^*B^* - \hat W \hat B)}_F^2$ implies the following bound, which holds with probability at least $1-\delta$:
\begin{align*}
    \norm{\Xcal (W^*B^* - \hat{W} \hat{B})}_F^2 \lesssim c^2\rbr{Mk \log \boldsymbol\kappa_\infty + d^2k\log \frac{\boldsymbol\kappa dT}{\delta}}.
\end{align*}
Since the smallest eigenvalue of the matrix $\Sigma_m = \sum_{t=0}^T X_m(t) X_m(t)'$ is at least $\underbar \lambda$ (\pref{thm:cov_conc}), we can convert the above prediction error bound to an estimation error bound and get
\begin{align*}
    \norm{W^*B^* - \hat W \hat B}_F^2 \lesssim {} & \frac{c^2\rbr{Mk \log \boldsymbol\kappa_\infty + d^2k\log \frac{\boldsymbol\kappa dT}{\delta}}}{\underbar \lambda},
\end{align*}
which implies the desired bound for the solution of \pref{eq:mt_opt}. 
%\begin{align*}
 %   \sum_{m=1}^M \norm{\hat A_m - A_m}_F^2 \lesssim \frac{c^2\rbr{Mk \log \boldsymbol\kappa_\infty + d^2k\log \frac{\boldsymbol\kappa dT}{\delta}}}{\underbar \lambda}.
%\end{align*}

\end{proof}

\section{Proof of \pref{thm:estimation_error_pert}}
\label{sec:estimation-error_pert_proof}
Here, we provide the key steps for bounding the average estimation error across the $M$ systems for the estimator in \pref{eq:mt_opt} in presence of misspecifications $D_m \in \R^{d \times d}$:
\begin{align*}
    A_m = \rbr{\sum_{i=1}^k \beta^*_m[i] W^*_i} + D_m,
\end{align*}
where we use $\zeta_m$ to denote the bound on misspecification in task $m$ and set $\bar \zeta^2 = \sum_{m=1}^M \zeta_m^2$. In the presence of misspecifications, we have $\Delta \coloneqq \tilde \Theta^* - \hat{\Theta} = VR + D$, where $V \in O^{d^2 \times 2k}$ is an orthonormal matrix, $R \in \R^{2k \times M}$, and $D \in \R^{d^2 \times M}$ is the misspecification error. As the analysis here shares its template with the proof of \pref{thm:estimation_error}, we provide a sketch with the complete details delegated to the extended version \cite{modi2021joint}. Same as in \pref{sec:estimation-error-proof}, we start with the fact that $(\hat W, \hat B)$ minimize the squared loss in \pref{eq:mt_opt}. However, in this case, we get an additional term caused by on the misspecifications $D_m$:
\begin{align}
\label{eq:sqloss_ineq_pert}
    &\frac{1}{2} \sum_{m=1}^M \norm{ \tilde{X}_m (W^*\beta^*_m - \hat{W} \hat{\beta}_m) }_2^2 & \nonumber \\
    \le & {} \sum_{m=1}^M \inner{\tilde \eta_m}{\tilde{X}_m \rbr{\hat{W}\hat{\beta}_m - W^*\beta^*_m}} \nonumber \\
    {} & {+ \sum_{m=1}^M 2\inner{ \tilde{X}_m\tilde{D}_m}{ \tilde{X}_m \rbr{\hat{W}\hat{\beta}_m - W^*\beta^*_m}}}.
\end{align}
We follow a similar proof strategy as in \pref{sec:estimation-error-proof} and account for the additional terms arising due to the misspecifications $D_m$. The error in the shared part, $\hat{W}\hat{\beta}_m - W^*\beta^*_m$, can still be rewritten as $Ur_m$ where $U \in \R^{d^2 \times 2k}$ is a matrix containing an orthonormal basis of size $2k$ in $\R^{d^2}$ and $r_m \in \R^{2k}$ is the system specific vector. We now show a decomposition similar to \pref{lem:breakup}:
\begin{lemma}
\label{lem:breakup_pert_app}
Under the misspecified shared linear basis structure in \pref{eq:linear_model_pert}, for any fixed orthonormal matrix $\bar{U} \in \R^{d^2 \times 2k}$, the low rank part of the total squared error can be decomposed as follows:
\begin{align}
    & \frac{1}{2}\sum_{m=1}^M \norm{\tilde{X}_m (W^*\beta^*_m - \hat{W} \hat{\beta}_m)}_F^2 & \nonumber\\ 
    \le {} & \sqrt{\sum_{m=1}^M \norm{\tilde\eta_m^\top \tilde{X}_m \bar{U}}_{\bar{V}^{-1}_m}^2}\sqrt{2\sum_{m=1}^M \norm{ \tilde{X}_m (W^*\beta^*_m - \hat{W} \hat{\beta}_m) }_2^2} \nonumber \\
    {} & + \sum_{m=1}^M \inner{\tilde\eta_m}{\tilde{X}_m (U-\bar{U}) r_m} & \nonumber \\
    {} & + \sqrt{\sum_{m=1}^M \norm{\tilde\eta_m^\top \tilde{X}_m \bar{U}}_{\bar{V}^{-1}_m}^2} \sqrt{2\sum_{m=1}^M \norm{\tilde{X}_m(\bar{U}-U)r_m}^2} \nonumber \\
    {} & {+ 2\sqrt{\bar \lambda}\bar{\zeta} \sqrt{\sum_{m=1}^M \norm{\tilde{X}_m \rbr{\hat{W}\hat{\beta}_m - W^*\beta^*_m}}_2^2}}. & \label{eq:app_pred_error_breakup_pert}
\end{align}
\end{lemma}

We bound each term on the RHS of~\eqref{eq:app_pred_error_breakup_pert} individually. Similar to \pref{sec:estimation-error-proof}, we choose the orthonormal $\R^{d^2 \times 2k}$ matrix $\bar U \in \Ncal_{\epsilon}$. Then, we use the following results, for which the proofs are provided in the longer version \cite{modi2021joint}.

\begin{proposition}[Bounding $\sum_{m=1}^M \norm{\tilde{X}_m(\bar{U}-U)r_m}^2$]
\label{prop:proj_cover_error_pert}
For the model in \pref{eq:linear_model_pert}, with probability at least $1-\delta_Z$, it holds that
\begin{small}
\begin{align}
    \sum_{m=1}^M \norm{\tilde{X}_m(\bar{U}-U)r_m}^2 
    \lesssim {} & \boldsymbol\kappa\epsilon^2 \rbr{MT\tr{C} +\sigma^2\log \tfrac{2}{\delta_Z}  {+ \bar{\lambda}\bar{\zeta}^2} }. \label{eq:pred_error_coarse_pert}
\end{align}
\end{small}
\end{proposition}

\begin{proposition}[Bounding $\sum_{m=1}^M \inner{\tilde\eta_m}{\tilde{X}_m (U-\bar{U}) r_m}$]
\label{prop:breakup_inner_prod_pert}
Under \pref{assum:noise} and  \pref{eq:linear_model_pert}, with probability at least $1-\delta_Z$ we have:
\begin{align}
\label{eq:breakup_inner_prod_pert}
    &\sum_{m=1}^M \inner{\tilde\eta_m}{\tilde{X}_m (U-\bar{U}) r_m}& \nonumber \\
    \lesssim {} & \sqrt{\boldsymbol\kappa} \epsilon \rbr{MT \tr{C} + \sigma^2\log \frac{1}{\delta_Z}} \nonumber \\
    {} & {+\sqrt{\boldsymbol\kappa \bar\lambda} \sqrt{MT \tr{C} + \sigma^2\log \frac{1}{\delta_Z}} \epsilon \bar \zeta}.
\end{align}
\end{proposition}

Finally, we are ready to put the above intermediate results together. Using the decomposition in \pref{lem:breakup_pert_app} and the term-wise upper bounds above, one can derive the desired estimation error rate. Below, we show the final steps with appropriate substitution for constants. A detailed proof can be found in \pref{app:estimation-error_pert}. 
% \textcolor{blue}{Aditya, can you double check here? It refers to arXiv version, but still provides the details below. Make sure it is clear to the reader how the flow goes here.}

As before, we substitute the termwise bounds from Propositions \ref{prop:proj_cover_error_pert}, \ref{prop:breakup_inner_prod_pert} and \ref{prop:mt-self-norm} in \pref{lem:breakup_pert_app} with values $\delta_U = \delta/3\abr{\Ncal}_\epsilon$, $\delta_C=\delta/3$ (in \pref{thm:cov_conc}), $\delta_Z = \delta/3$. Noting that $k \le d^2$ and $\log \frac{1}{\delta} \lesssim d^2k\log \frac{k}{\delta \epsilon}$, by setting $\epsilon = \frac{k}{\sqrt{\boldsymbol\kappa}d^2T}$ we finally get the following quadratic inequality in the error term $\Xi \coloneqq \norm{\Xcal (W^*B^* - \hat{W} \hat{B})}_F$:
\begin{align*}
    \frac{1}{2} \Xi^2 \lesssim {} & \rbr{\sqrt{c^2\rbr{Mk \log \boldsymbol\kappa_\infty + d^2k \log \frac{\boldsymbol\kappa dT}{\delta}}}{+ \sqrt{\bar \lambda}\bar{\zeta}}} \Xi\\
    {} & + c^2\rbr{Mk \log \boldsymbol\kappa_\infty + \frac{d^2k}{T}\log \frac{\boldsymbol\kappa dT}{\delta}}  \\
    {} & {+ c \sqrt{\frac{\bar \lambda \bar \zeta^2}{T}\rbr{Mk \log \boldsymbol\kappa_\infty + \frac{d^2k}{T}\log \frac{\boldsymbol\kappa dT}{\delta}}}}.
\end{align*}
The quadratic inequality for the prediction error $\norm{\Xcal(W^*B^* - \hat W \hat B)}_F^2$ implies the following bound with probability at least $1-\delta$:
\begin{align*}
    \Xi^2 \lesssim {} & c^2\rbr{Mk \log \boldsymbol\kappa_\infty + d^2k\log \frac{\boldsymbol\kappa dT}{\delta}} +  {\bar \lambda \bar \zeta^2}.
\end{align*}
Since $\underbar \lambda = \min_m \underbar \lambda_m$, an estimation error bound for the solution of \pref{eq:mt_opt}:
\begin{align*}
    & \sum_{m=1}^M \norm{\hat A_m - A_m}_F^2 & \\
    \lesssim {} & \frac{c^2\rbr{Mk \log \boldsymbol\kappa_\infty + d^2k\log \frac{\boldsymbol\kappa dT}{\delta}}}{\underbar \lambda} +  {(\boldsymbol\kappa_\infty + 1) \bar \zeta^2}.
\end{align*}

\section{Auxiliary Probabilistic Inequalities}
\label{sec:aux_results}
In this section, we state the general probabilistic inequalities which we used in proving the main results in the previous sections. The proofs for these results can be found in \pref{app:gen_ineq}.

\begin{proposition}[Bounding the noise sequence]
\label{prop:truncation}
For $T = 0,1,\ldots$, and $0<\delta<1$, let $\Ecal_{\mathrm{bdd}}$ be the event 
\begin{align}
    \Ecal_{\mathrm{bdd}}(\delta) \coloneqq \cbr{\max_{1 \le t \le T, m \in [M]} \norm{\eta_m(t)}_\infty \le \sqrt{2\sigma^2 \log \tfrac{2dMT}{\delta}}}.
\end{align}
Then, we have $\PP[\Ecal_{\mathrm{bdd}}] \ge 1-\delta$. For simplicity, we denote the {above} upper-bound by $b_T(\delta)$.
\end{proposition}

\begin{proposition}[Noise covariance concentration]
\label{prop:noise-conc}
For $T$ and $0 < \delta < 1$, let $\Ecal_{\eta}$ be the event
\begin{align*}
   \Ecal_{\eta}(\delta) \coloneqq  \cbr{\tfrac{3\lambda_{\min}(C)}{4} I \preceq \tfrac{1}{T}\sum_{t=1}^T \eta_m(t) \eta_m(t)' \preceq \tfrac{5\lambda_{\max}(C)}{4} I}.
\end{align*}
Then, if $T \ge T_{\eta}(\delta) \coloneqq \frac{c_\eta d\sigma^2 }{\lambda_{\min}(C)^2}\log 18/\delta$, we have $\PP[\Ecal_{\mathrm{bdd}}(\delta) \cap \Ecal_{\eta}(\delta)] \ge 1-2\delta$.
\end{proposition}

Define $Z \in \RR^{dT \times M}$ as the pooled noise matrix as follows:
\begin{align}
\label{eq:tot_noise}
    Z = \sbr{\tilde\eta_1(T) | \tilde\eta_2(T) \cdots | \tilde \eta_M(T)},
\end{align}
with each column vector $\eta_m(T) \in \R^{dT}$ as the concatenated noise vector $(\eta_m(1), \eta_m(2), \ldots, \eta_m(T))$ for the $m$-th system.
\begin{proposition}[Bounding total magnitude of noise]
\label{prop:noise-magnitude}
For the joint noise matrix $Z \in \R^{dT \times M}$ defined in \pref{eq:tot_noise}, with probability at least $1-\delta$, we have:
\begin{align*}
    \norm{Z}_F^2 \le MT\tr{C} + \log \frac{2}{\delta}.
\end{align*}
We denote the above event by $\Ecal_Z(\delta)$.
\end{proposition}

The following result shows a self-normalized martingale bound for vector valued noise processes.
\begin{proposition}
\label{prop:matrix-self-norm}
For the system in \pref{eq:mt-lti}, for any $0<\delta<1$ and system $m \in [M]$, with prob. at least $1-\delta$, we have: 
\begin{align*}
    &\opnorm{\bar{V}_m^{-1/2}(T-1) \sum_{t=0}^{T-1} x_m(t) \eta_m(t+1)'}{2} & \\
    \le {} & \sigma \sqrt{8d \log \rbr{\frac{5\det \rbr{\bar{V}_m(T-1)}^{1/2d} \det \rbr{V}^{-1/2d} }{\delta^{1/d}}}},&
\end{align*}
where $\bar{V}_{m}(s) = \sum_{t=0}^s x_m(t) x_m(t)' + V$ and $V$ is a deterministic positive definite matrix. 
\end{proposition}

\begin{lemma}[Covering low-rank matrices \cite{du2020few}]
\label{lem:lowrank-cover}
For the set of orthonormal matrices $O^{d \times d'}$ (with $d > d'$), there exists $\Ncal_\epsilon \subset O^{d \times d'}$ that forms an $\epsilon$-net of $O^{d \times d'}$ in Frobenius norm such that $|\Ncal_\epsilon| \le (\tfrac{6\sqrt{d'}}{\epsilon})^{dd'}$, i.e., for every $V \in O^{d \times d'}$, there exists $V' \in \Ncal_\epsilon$ and $\|V - V'\|_F \le \epsilon$.
\end{lemma}

\section{Concluding Remarks}
\label{sec:conc}
We studied the problem of jointly learning multiple linear time-invariant dynamical systems, under the assumption that their transition matrices can be expressed based on an unknown shared basis. Our finite-time analysis for the proposed joint estimator shows that pooling data across systems can provably improve over individual estimators, even in presence of moderate misspecifications. The results highlight the critical roles of the spectral properties of the system matrices and the number of the basis matrices, in the efficiency of joint estimation. Further, we characterize fundamental differences between joint estimation of system dynamics using dependent state trajectories and learning from independent stationary observations. Considering different shared structures, extensions of the presented results to explosive systems, or those with high-dimensional transition matrices, as well as joint learning of multiple non-linear dynamical systems, all are interesting avenues for future work that this paper paves the road towards.

{\section*{Acknowledgements}
The authors appreciate the helpful comments of the reviewers on the initial version of this paper. 
%This work was completed when AM was at Univ. of Michigan. 
% Part of this work was done while AM, MKSF and AT were participating in the `Theory of RL' program at the Simons Institute for the Theory of Computing. 
AM is supported in part by a grant from the Open Philanthropy Project to the CHAI and NSF CAREER IIS-1452099.}

% \subsubsection*{References}
\bibliography{aut_main_arxiv.bib}

\newpage

\onecolumn
% \hsize\textwidth
% \def\toptitlebar{
% \hrule height3pt
% \vskip .25in}

% \def\bottomtitlebar{
% \vskip .25in
% \hrule height1pt
% \vskip .25in}
% % \linewidth\hsize \toptitlebar 
% {\centering{\Large\bfseries Appendices\par}}
% % \bottomtitlebar %\vskip 0.2in plus 1fil minus 0.1in

\appendix

\section{Additional Numerical Simulations}
\label{app:add_exp}
In order to completely visualize the dependence on all the parameters ($k, d$ and $T$) other than the number of systems $M$, one needs to vary all parameters ($k,d$ and $\bar \zeta$) at different rates and plot the estimation errors. Such extensive empirical analyses do not constitute the focus of this paper, and are indeed studied in the existing literature of learning one system individually. For example, the dependence on $T$ for the estimation error is well understood and cannot be better than $1/\sqrt{T}$ \cite{faradonbeh2018finite,sarkar2019near,simchowitz2018learning}.

Moreover, we further verify that different random matrices for the shared linear basis lead to similar joint learning curves. We simulate the experiments by using random matrices whose entries are sampled from a uniform distribution over the range $(-2.0,2.0)$ (followed by normalization steps to obtain the desired stable or unit-root dynamics). The results, reported in \pref{fig:unif-plots} below, show the same benefits of joint learning compared to individual estimation, as shown in the paper.

\begin{figure}[ht]
\centering     %%% not \center
\subfigure[System matrices $A_m$ are stable]{\label{fig:unif-a}\includegraphics[width=0.45\columnwidth]{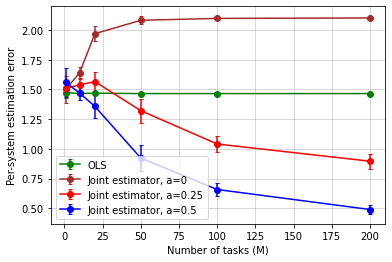}}
\subfigure[System matrices $A_m$ have a unit root]{\label{fig:unif-b}\includegraphics[width=0.5\columnwidth]{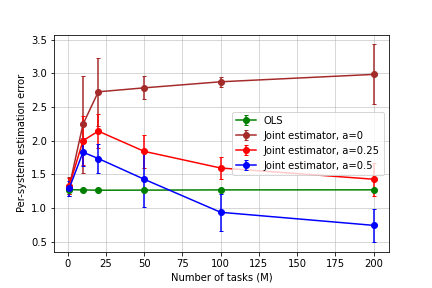}}
\caption{Per-system estimation errors vs. the number of systems $M$, for the proposed joint learning method and individual least-squares estimates of the linear dynamical systems.}\label{fig:unif-plots}
\end{figure}

\section{Proofs of Auxiliary Results}
\label{app:gen_ineq}
Here, we give proofs of the probabilistic inequalities and intermediate results in \pref{sec:aux_results}.
\paragraph*{Proposition}[Restatement of \pref{prop:truncation}] For $T = 0,1,\ldots$, and $0<\delta<1$, let $\Ecal_{\mathrm{bdd}}$ be the event 
\begin{align}
    \Ecal_{\mathrm{bdd}}(\delta) \coloneqq \cbr{\max_{1 \le t \le T, m \in [M]} \norm{\eta_m(t)}_\infty \le \sqrt{2\sigma^2 \log \tfrac{2dMT}{\delta}}}.
\end{align}
Then, we have $\PP[\Ecal_{\mathrm{bdd}}] \ge 1-\delta$. For simplicity, we denote the above upper-bound by $b_T(\delta)$.
\begin{proof}
Let $e_i$ be the $i$-th member of the standard basis of $\mathbb{R}^d$. Using the sub-Gaussianity of the random vector $\eta_{m}(t)$ given the sigma-field $\Fcal_{t-1}$, we have
\begin{align*}
    \PP\sbr{\abr{\inner{e_i}{\eta_m(t)}} > \sqrt{2\sigma^2 \log \frac{2}{\delta'}}} \le \delta'.
\end{align*}
Therefore, taking a union bound over all basis vectors $i=1,\cdots, d$, all systems $m \in [M]$, and all time steps $t=1 ,\cdots, T$ that $\eta_m(t) > \sqrt{2\sigma^2 \log \frac{2}{\delta'}}$, we get the desired result by letting $\delta'=\delta (dMT)^{-1}$.
\end{proof}

\paragraph*{Proposition}[Restatement of \pref{prop:noise-conc}] For $T$ and $0 < \delta < 1$, let $\Ecal_{\eta}$ be the event
\begin{align*}
   \Ecal_{\eta}(\delta) \coloneqq  \cbr{\tfrac{3\lambda_{\min}(C)}{4} I \preceq \tfrac{1}{T}\sum_{t=1}^T \eta_m(t) \eta_m(t)' \preceq \tfrac{5\lambda_{\max}(C)}{4} I}.
\end{align*}
Then, if $T \ge T_{\eta}(\delta) \coloneqq \frac{c_\eta d\sigma^2 }{\lambda_{\min}(C)^2}\log 18/\delta$, we have $\PP[\Ecal_{\mathrm{bdd}}(\delta) \cap \Ecal_{\eta}(\delta)] \ge 1-2\delta$.
\begin{proof}
Here, we will bound the largest eigenvalue of the deviation matrix $\sum_{t=1}^T \eta_m(t)\eta_m(t)' - TC$. For the spectral norm of this matrix, using Lemma 5.4 from Vershynin \cite{vershynin2018high}, we have:
\begin{align*}
    \opnorm{\sum_{t=1}^T \eta_m(t)\eta_m(t)' - TC}{2} \le \frac{1}{1-2\tau}\sup_{v \in \Ncal_{\tau}} \abr{v' \rbr{\sum_{t=1}^T \eta_m(t)\eta_m(t)' - TC} v} ,
\end{align*}
where $\Ncal_{\tau}$ is a $\tau$-cover of the unit sphere $\Scal^{d-1}$. Now, it holds that $\abr{\Ncal_{\tau}} \le \rbr{1 + 2/\tau}^d$. Thus, we get:
\begin{align*}
    \PP\sbr{\opnorm{\sum_{t=1}^T \eta_m(t)\eta_m(t)' - TC}{2} \ge \epsilon} \le {} & \PP\sbr{\sup_{v \in \Ncal_{\tau}} \abr{v' \rbr{\sum_{t=1}^T \eta_m(t)\eta_m(t)' - TC} v } \ge (1-2\tau) \epsilon}.
\end{align*}
Using some martingale concentration arguments, we first bound the probability on the RHS for a fixed vector $v \in \Ncal_{\tau}$. Then, taking a union bound over all $v \in \Ncal_{\tau}$ will lead to the final result. 

For a given $t$, since $\eta_m(t)'v$ is conditionally sub-Gaussian with parameter $\sigma$, the quantity $v'\eta_m(t)\eta_m(t)'v-v'Cv$ is a conditionally sub-exponential martingale difference. Using Theorem 2.19 of Wainwright \cite{wainwright2019high}, for small values of $\epsilon$, we have
\begin{align*}
    \PP\sbr{ \abr{v' \rbr{\sum_{t=1}^T \eta_m(t)\eta_m(t)' - TC} v} \ge (1-2\tau)\epsilon} \le 2 \exp\rbr{-\frac{c_\eta(1-2\tau)^2\epsilon^2}{T\sigma^2}} ,
\end{align*}
where $c_\eta$ is some universal constant. Taking a union bound, setting total failure probability to $\delta$, and letting $\tau = 1/4$, we obtain that with probability at least $1-\delta$, it holds that
\begin{align*}
    \lambda_{\max}\rbr{\sum_{t=1}^T \eta_m(t)\eta_m(t)' - TC} \le c_\eta\sigma \sqrt{T\log \rbr{\frac{2 \cdot 9^d}{\delta}}}.
\end{align*}
According to Weyl's inequality, for $T \ge T_{\eta}(\delta) \coloneqq \frac{c_\eta d\sigma^2 }{\lambda_{\min}(C)^2}\log 18/\delta$, we have:
\begin{align*}
    \frac{3\lambda_{\min}(C)}{4}I \preceq \frac{1}{T}\sum_{t=1}^T \eta_m(t) \eta_m(t)' \preceq \frac{5\lambda_{\max}(C)}{4}I.
\end{align*}
\end{proof}

\paragraph*{Proposition}[Restatement of \pref{prop:noise-magnitude}] For the joint noise matrix $Z \in \R^{dT \times M}$ defined in \pref{eq:tot_noise}, with probability at least $1-\delta$, we have:
\begin{align*}
    \norm{Z}_F^2 \le MT\tr{C} + \log \frac{2}{\delta}.
\end{align*}
We denote the above event by $\Ecal_Z(\delta)$.
\begin{proof}
For each system $m$, we know that $\EE[\eta_m(t)[i]^2|\Fcal_{t-1}] = C_{ii}^2$. Similar to the previous proof, we know that $\eta_m(t)[i]^2$ follows a conditionally sub-exponential distribution given $\Fcal_{t-1}$. Using the sub-exponential bound for martingale difference sequences, for large enough $T$ we get:
\begin{align*}
    \norm{Z}_F^2 = \sum_{m=1}^M \sum_{t=1}^T \norm{\eta_m(t)}^2 \le MT\tr{C} + \log \frac{2}{\delta},
\end{align*}
with probability at least $1-\delta$.
\end{proof}

Next, we show the proof of \pref{prop:matrix-self-norm} based on the following well-known probabilistic inequalities below. 
The first inequality is a concentration bound for self-normalized martingales, which can be found in Lemma 8 and Lemma 9 in the work of Abbasi-Yadkori et al. \cite{abbasi2011improved}. More details about self-normalized process can be found in the work of Victor et al. \cite{victor2009self}.
\begin{lem}
\label{lem:self-norm}
Let $\{\Fcal_t\}_{t=0}^\infty$ be a filtration. Let $\{\eta_t\}_{t=1}^\infty$ be a real valued stochastic process such that $\eta_t$ is $\Fcal_t$ measurable and $\eta_t$ is conditionally $\sigma$-sub-Gaussian for some $R > 0$, i.e.,
\begin{align*}
    \forall \lambda \in \R, \EE\sbr{\exp(\lambda \eta_t) | \Fcal_{t-1}} \le e^{\frac{\lambda^2 \sigma^2}{2}}
\end{align*}
Let $\{X_t\}_{t=1}^\infty$ be an $\R^d$-valued stochastic process such that $X_t$ is $\Fcal_t$ measurable. Assume that $V$ is a $d \times d$ positive definite matrix. For any $t \ge 0$, define 
\begin{align*}
    \bar{V}_t = V + \sum_{s=1}^t X_s X_s', \quad S_t = \sum_{s=1}^t \eta_{s+1}X_s.
\end{align*}
Then with probability at least $1 - \delta$, for all $t \ge 0$ we have
\begin{align*}
    \norm{S_t}_{\bar{V}_t^{-1}}^2 \le 2\sigma^2 \log \rbr{\frac{\det \rbr{\bar{V}_t}^{1/2} \det \rbr{V}^{-1/2}}{\delta}}.
\end{align*}
\end{lem}

The second inequality is the following discretization-based bound shown in Vershynin \cite{vershynin2018high} for random matrices:
\begin{proposition}
\label{prop:matrix-l2}
Let $M$ be a random matrix. For any $\epsilon<1$, let $\Ncal_{\epsilon}$ be an $\epsilon$-net of $\Scal^{d-1}$ such that for any $w \in \Scal^{d-1}$, there exists $\bar w \in \Ncal_{\epsilon}$ with $\norm{w-\bar w} \le \epsilon$. Then for any $\epsilon < 1$, we have: 
\begin{align*}
    \PP\sbr{\opnorm{M}{2} > z} \le \PP\sbr{\max_{\bar w \in \Ncal_\epsilon} \norm{M\bar w} > (1-\epsilon)z}. 
\end{align*}
\end{proposition}

With the aforementioned results, we now show the proof of \pref{prop:matrix-self-norm}.
\paragraph*{Proposition}[Restatement of \pref{prop:matrix-self-norm}] For the system in \pref{eq:mt-lti}, for any $0<\delta<1$ and system $m \in [M]$, with probability at least $1-\delta$, we have: 
\begin{align*}
    \opnorm{\bar{V}_m^{-1/2}(T-1) \sum_{t=0}^{T-1} x_m(t) \eta_m(t+1)'}{2} \le {} & \sigma \sqrt{8d \log \rbr{\frac{5\det \rbr{\bar{V}_m(T-1)}^{1/2d} \det \rbr{V}^{-1/2d} }{\delta^{1/d}}}},&
\end{align*}
where $\bar{V}_{m}(s) = \sum_{t=0}^s x_m(t) x_m(t)' + V$ and $V$ is a deterministic positive definite matrix.
\begin{proof}
For the partial sum $S_m(t) = \sum_{s=0}^t x_m(s) \eta_m(s+1)'$, using \pref{prop:matrix-l2} with $\epsilon=1/2$, we get:
\begin{align*}
\PP\sbr{\opnorm{\bar{V}_m^{-1/2}(T-1) S_m(T-1)}{2} \ge y} \le \sum_{\bar w \in \Ncal_{\epsilon}} \PP\sbr{\norm{\bar{V}_m^{-1/2}(T-1) S_m(T-1)\bar w}^2 \ge \frac{y^2}{4}},
\end{align*}
where $\bar w$ is a fixed unit norm vector in $\Ncal_{\epsilon}$. We can now apply \pref{lem:self-norm} with the $\sigma$ sub-Gaussian noise sequence $\eta_t' w$ to get the final high probability bound.
\end{proof}

\section{Remaining Proofs from \pref{sec:estimation-error-proof}}
\label{app:estimation-error}
In this section, we provide a detailed analysis of the average estimation error across the $M$ systems for the estimator in \pref{eq:mt_opt}. We start with the proof of \pref{lem:breakup} which bounds the prediction error for all systems $m \in [M]$:
\paragraph*{Lemma}[Restatement of \pref{lem:breakup}]
For any fixed orthonormal matrix $\bar{U} \in \R^{d^2 \times 2k}$, the total squared prediction error in \pref{eq:mt_opt} for $(\hat W, \hat B)$ can be decomposed as follows:
\begin{align}
    & \frac{1}{2}\sum_{m=1}^M \norm{\tilde{X}_m (W^*\beta^*_m - \hat{W} \hat{\beta}_m)}_F^2 & \nonumber\\ 
    \le {} & \sqrt{\sum_{m=1}^M \norm{\tilde\eta_m^\top \tilde{X}_m \bar{U}}_{\bar{V}^{-1}_m}^2}\sqrt{2\sum_{m=1}^M \norm{ \tilde{X}_m (W^*\beta^*_m - \hat{W} \hat{\beta}_m) }_2^2} + \sum_{m=1}^M \inner{\tilde\eta_m}{\tilde{X}_m (U-\bar{U}) r_m} & \nonumber \\
    {} & + \sqrt{\sum_{m=1}^M \norm{\tilde\eta_m^\top \tilde{X}_m \bar{U}}_{\bar{V}^{-1}_m}^2} \sqrt{\sum_{m=1}^M \norm{\tilde{X}_m(\bar{U}-U)r_m}^2} . & \label{eq:app_pred_error_breakup}
\end{align}
\begin{proof}
We first define $\tilde{\Sigma}_{m,\up}$ and $\tilde{\Sigma}_{m,\dn}$ as the block diagonal matrices in $\R^{d^2 \times d^2}$, with each $d\times d$ block of $\tilde{\Sigma}_{m,\up}$ and $\tilde{\Sigma}_{m,\dn}$ containing $\bar{\Sigma}_m$ and $\underbar{\Sigma}_m$, respectively. Let $V_m = U^\top \tilde{\Sigma}_m U + U^\top\tilde{\Sigma}_{m,\dn}U$ be the regularized covariance matrix of projected covariates $\tilde{X}_m U$. For any orthonormal matrix $\bar{U} \in \R^{d^2 \times 2k}$, we define $\bar V_m = \bar U^\top \tilde{\Sigma}_m \bar U + \bar U^\top\tilde{\Sigma}_{m,\dn}\bar U$, and proceed as follows:
\begin{align}
    \frac{1}{2}\sum_{m=1}^M \norm{ \tilde{X}_m (W^*\beta^*_m - \hat{W} \hat{\beta}_m) }_2^2 = {} & \sum_{m=1}^M \inner{\tilde \eta_m}{\tilde{X}_m \bar{U} r_m} + \sum_{m=1}^M \inner{\tilde\eta_m}{\tilde{X}_m (U-\bar{U}) r_m}  \label{eq:coverU}\\
    \le {} & \sum_{m=1}^M \norm{\tilde\eta_m^\top \tilde{X}_m \bar{U}}_{\bar{V}^{-1}_m}\norm{r_m}_{\bar{V}_m} + \sum_{m=1}^M \inner{\tilde\eta_m}{\tilde{X}_m (U-\bar{U}) r_m}  \label{eq:c-s}\\ 
    \le {} & \sum_{m=1}^M \norm{\tilde\eta_m^\top \tilde{X}_m \bar{U}}_{\bar{V}^{-1}_m}\norm{r_m}_{{V}_m} + \sum_{m=1}^M \inner{\tilde\eta_m}{\tilde{X}_m (U-\bar{U}) r_m} \nonumber \\
    {} & + \sum_{m=1}^M \norm{\tilde\eta_m^\top \tilde{X}_m \bar{U}}_{\bar{V}^{-1}_m}\rbr{\norm{r_m}_{\bar{V}_m}-\norm{r_m}_{{V}_m}} \label{eq:add-subtract}\\
    \le {} & \sqrt{\sum_{m=1}^M \norm{\tilde\eta_m^\top \tilde{X}_m \bar{U}}_{\bar{V}^{-1}_m}^2}\sqrt{\sum_{m=1}^M \norm{r_m}^2_{V_m}} + \sum_{m=1}^M \inner{\tilde\eta_m}{\tilde{X}_m (U-\bar{U}) r_m} \nonumber \\
    {} & + \sqrt{\sum_{m=1}^M \norm{\tilde\eta_m^\top \tilde{X}_m \bar{U}}_{\bar{V}^{-1}_m}^2} \sqrt{\sum_{m=1}^M \rbr{\norm{r_m}_{\bar{V}_m} - \norm{r_m}_{V_m}}^2} \label{eq:c-s2} . 
\end{align}
The first equality in \pref{eq:coverU} uses the fact that the error matrix $\tilde \Theta^* - \hat \Theta$ is of rank at most $2k$ and then introduces a matrix $\bar U$ leading to the two terms on the RHS. In inequality \pref{eq:c-s}, we use Cauchy-Schwartz inequality to bound the first term with respect to the norm induced by matrix $\bar{V}_m$. The next step \pref{eq:add-subtract} again follows by simple algebra where rewrite the term $\norm{r_m}_{\bar{V}_m}$ as $\norm{r_m}_{{V}_m} + (\norm{r_m}_{\bar{V}_m} -\norm{r_m}_{{V}_m} )$ and collect the terms accordingly. Finally, in the last step \pref{eq:c-s2}, we again use the Cauchy-Schwarz inequality to rewrite the first and last terms from the previous step. 

Now, note that $\norm{r}^2_{V_m} = \norm{Ur}_{\tilde \Sigma_m}$. Thus, for the last term in the RHS of \pref{eq:c-s2}, we can use the reverse triangle inequality for any two vectors $a,b \in \R^{2k}$, $\abr{\norm{a}-\norm{b}} \le \norm{a-b}$ with $a=\bar U r_m$ and $b=U r_m$. Hence, we have:
\begin{align}
    \frac{1}{2}\sum_{m=1}^M \norm{ \tilde{X}_m (W^*\beta^*_m - \hat{W} \hat{\beta}_m) }_2^2 \le {} & \sqrt{\sum_{m=1}^M \norm{\tilde\eta_m^\top \tilde{X}_m \bar{U}}_{\bar{V}^{-1}_m}^2}\sqrt{\sum_{m=1}^M \norm{Ur_m}^2_{\tilde{\Sigma}_m + \tilde\Sigma_{m,\dn}}} + \sum_{m=1}^M \inner{\tilde\eta_m}{\tilde{X}_m (U-\bar{U}) r_m} \nonumber \\
    {} & + \sqrt{\sum_{m=1}^M \norm{\tilde\eta_m^\top \tilde{X}_m \bar{U}}_{\bar{V}^{-1}_m}^2} \sqrt{\sum_{m=1}^M \norm{(\bar U - U)r_m}_{\tilde{\Sigma}_m + \tilde\Sigma_{m,\dn}}^2}
\end{align}

Further, since $\tilde \Sigma_m \succeq \tilde \Sigma_{m,\dn}$ (\pref{thm:cov_conc}), we have $r'U' \rbr{\tilde \Sigma_m - \tilde \Sigma_{m,\dn}} Ur \ge 0$ for all $r\in \RR^{2k}$. Thus, we can rewrite the previous inequality as:
\begin{align*}
    \frac{1}{2}\sum_{m=1}^M \norm{ \tilde{X}_m (W^*\beta^*_m - \hat{W} \hat{\beta}_m) }_2^2 \le {} & \sqrt{\sum_{m=1}^M \norm{\tilde\eta_m^\top \tilde{X}_m \bar{U}}_{\bar{V}^{-1}_m}^2}\sqrt{2\sum_{m=1}^M \norm{Ur_m}^2_{\tilde{\Sigma}_m}} + \sum_{m=1}^M \inner{\tilde\eta_m}{\tilde{X}_m (U-\bar{U}) r_m} \\
    {} & + \sqrt{\sum_{m=1}^M \norm{\tilde\eta_m^\top \tilde{X}_m \bar{U}}_{\bar{V}^{-1}_m}^2} \sqrt{2\sum_{m=1}^M \norm{\tilde{X}_m(\bar{U}-U)r_m}^2}  \\
    \le {} & \sqrt{\sum_{m=1}^M \norm{\tilde\eta_m^\top \tilde{X}_m \bar{U}}_{\bar{V}^{-1}_m}^2}\sqrt{2\sum_{m=1}^M \norm{ \tilde{X}_m (W^*\beta^*_m - \hat{W} \hat{\beta}_m) }_2^2}\\
    {} & + \sum_{m=1}^M \inner{\tilde\eta_m}{\tilde{X}_m (U-\bar{U}) r_m} \\
    {} & + \sqrt{\sum_{m=1}^M \norm{\tilde\eta_m^\top \tilde{X}_m \bar{U}}_{\bar{V}^{-1}_m}^2} \sqrt{2\sum_{m=1}^M \norm{\tilde{X}_m(\bar{U}-U)r_m}^2}.
\end{align*}
In the second inequality, we simply expand the term $\norm{Ur_m}^2_{\tilde{\Sigma}_m}$ using the relation $\hat W \hat \beta_m - W^* \beta_m^* = Ur_m$.
\end{proof}

% \vspace{-2em}
We will now give a detailed proof of the bound of each term on the rhs of \pref{eq:pred_error_breakup}. As stated in the main text, we select the matrix $\bar{U}$ to be an element of $\Ncal_{\epsilon}$ (cover of set of orthonormal matrices in $\R^{d^2 \times 2k}$) such that $\norm{\bar{U} - U}_F \le \epsilon$.
% \vspace{-2em}
\paragraph*{Proposition}[Restatement of Proposition \ref{prop:proj_cover_error}]
% \label{prop:proj_cover_error}
Under \pref{assum:linear_model}, for the noise process $\{\eta_m(t)\}_{t=1}^\infty$ defined for each system, with probability at least $1-\delta_Z$, we have:
\begin{align}
    \sum_{m=1}^M \norm{\tilde{X}_m(\bar{U}-U)r_m}^2 \lesssim {} & \boldsymbol\kappa\epsilon^2 \rbr{MT\tr{C} +\sigma^2\log \frac{2}{\delta_Z} }. \label{eq:app_pred_error_coarse}
\end{align}
\begin{proof}
In order to bound the term above, we use the squared loss inequality in \pref{eq:sqloss_ineq} as follows: 
\begin{align*}
    \norm{\Xcal (W^*B^* - \hat{W} \hat{B})}_F^2 \le 2 \inner{Z}{\Xcal (W^*B^* - \hat{W} \hat{B})} \le 2\norm{Z}_F \norm{\Xcal (W^*B^* - \hat{W} \hat{B})}_F,
\end{align*}
which leads to the inequality $\norm{\Xcal (W^*B^* - \hat{W} \hat{B})}_F \le 2 \norm{Z}_F $. Using the concentration result in \pref{prop:noise-magnitude}, with probability at least $1-\delta_Z$, we get
\begin{align*}
    \|Z\|_F \lesssim \sqrt{MT\tr{C} + \sigma^2\log \frac{1}{\delta_Z}}.
\end{align*}
Thus, we have $\norm{\Xcal (W^*B^* - \hat{W} \hat{B})}_F \lesssim 2 \sqrt{MT\tr{C} + \sigma^2\log \frac{1}{\delta_Z}}$, with probability at least $1-\delta_Z$ which gives:
\begin{align*}
    \sum_{m=1}^M \norm{\tilde{X}_m(\bar{U}-U)r_m}^2 \le {} & \sum_{m=1}^M \norm{\tilde{X}_m}^2\norm{\bar{U}-U}^2 \norm{r_m}^2 \le {}\sum_{m=1}^M \bar{\lambda}_m \epsilon^2 \norm{r_m}^2 \\
    = {} & \sum_{m=1}^M \bar{\lambda}_m \epsilon^2 \norm{Ur_m}^2 
    \le {}\sum_{m=1}^M \bar{\lambda}_m \epsilon^2 \frac{\norm{\tilde{X}_m U r_m}^2}{\underbar{\lambda}_m} \\
    \le {} & \boldsymbol\kappa \epsilon^2 \sum_{m=1}^M \norm{\tilde{X}_m U r_m}^2  = {} \boldsymbol\kappa \epsilon^2 \sum_{m=1}^M \norm{\Xcal (W^*B^* - \hat{W} \hat{B})}_F^2  \\
    \lesssim {} & \boldsymbol\kappa\epsilon^2 \rbr{MT\tr{C} +\sigma^2\log \frac{1}{\delta_Z}}.
\end{align*}
\end{proof}

\vspace{-3em}
\paragraph*{Proposition}[Restatement of Proposition~\ref{prop:breakup_inner_prod}]
% \label{prop:breakup_inner_prod}
Under \pref{assum:noise} and \pref{assum:linear_model}, with probability at least $1-\delta_Z$, we have:
\begin{align}
\label{eq:app_breakup_inner_prod}
    \sum_{m=1}^M \inner{\tilde\eta_m}{\tilde{X}_m (U-\bar{U}) r_m} \lesssim {} & \sqrt{\boldsymbol\kappa} \epsilon \rbr{MT \tr{C} + \sigma^2\log \frac{1}{\delta_Z}}.
\end{align}
\begin{proof}
Using Cauchy-Schwarz inequality and \pref{prop:proj_cover_error}, we bound the term as follows:
\begin{align*}
    \sum_{m=1}^M \inner{\tilde\eta_m}{\tilde{X}_m (U-\bar{U}) r_m} \le {} & \sqrt{\sum_{m=1}^M \norm{\tilde{\eta}_m}^2} \sqrt{\sum_{m=1}^M \norm{\tilde{X}_m(\bar{U}-U)r_m}^2} \\
    \lesssim {} & \sqrt{MT \tr{C} + \sigma^2\log \frac{1}{\delta_Z}} \sqrt{\boldsymbol\kappa\epsilon^2 \rbr{MT\tr{C} +\sigma^2\log \frac{1}{\delta_Z} }}\\
    \lesssim {} & \sqrt{\boldsymbol\kappa} \epsilon \rbr{MT \tr{C} + \sigma^2\log \frac{1}{\delta_Z}}.
\end{align*}
\end{proof}

\paragraph*{Proposition}[Restatement of Proposition~\ref{prop:mt-self-norm}]
% \label{prop:mt-self-norm}
For an arbitrary orthonormal matrix $\bar{U} \in \R^{d^2 \times 2k}$ in the $\epsilon$-cover $\Ncal_\epsilon$ defined in \pref{lem:lowrank-cover}, let $\Sigma \in \R^{d^2 \times d^2}$ be a positive definite matrix, and define $S_{m}(\tau) = \tilde \eta_m(\tau)^\top \tilde{X}_m(\tau) \bar{U}$, $\bar{V}_m(\tau) = \bar{U}'\rbr{\tilde{\Sigma}_m(\tau) + \Sigma}\bar{U}$, and $V_0 = \bar{U}'\Sigma\bar{U}$. Then, %for filtration $\Fcal_t$ and all $t \ge 0$, 
consider the following event: 
\begin{align*}
    \Ecal_{1}(\delta_U) \coloneqq \cbr{\omega \in \Omega: \sum_{m=1}^M \norm{S_{m}(T)}_{\bar{V}_m^{-1}(T)}^2 \le 2\sigma^2\log \rbr{\frac{\Pi_{m=1}^M\det(\bar{V}_m(T))\det (V_0)^{-1}}{\delta_U}}}.
\end{align*}
For $\Ecal_{1}(\delta_U)$, we have:
\begin{align}
\label{eq:app_mt-self-norm}
    \PP\sbr{\Ecal_1(\delta_U)} \ge 1-\rbr{\frac{6\sqrt{2k}}{\epsilon}}^{2d^2k}\delta_U.
\end{align}
\paragraph*{Proof of \pref{prop:mt-self-norm}}
First, using the vectors $\tilde{x}_{m,j}(t)$ defined in \pref{sec:joint-learning}, for the matrix $\bar{U}$, define $\bar{x}_{m,j}(t) = \bar{U}' \tilde{x}_{m,j}(t) \in \R^{2k}$. It is straightforward to see that $\bar{V}_m(t) = \sum_{s=1}^t \sum_{j=1}^d \bar{x}_{m,j}(s) \bar{x}_{m,j}(s)' + V_0$.

Now, we show that the result can essentially be stated as a corollary of the following result for a univariate regression setting:

\begin{lem}[Lemma 2 of \cite{hu2021near}]
Consider a fixed matrix $\bar U \in \RR^{p \times 2k}$ and let $\bar V_{m}(t) = \bar U' (\sum_{s=0}^t x_m(s)x_m(s)') \bar U + \bar U' V_0 \bar U$. Consider a noise process $w_{m}(t+1) \in \RR$ adapted to the filtration $\Fcal_t = \sigma(w_m(1),\ldots,w_m(t), X_m(1), \ldots, X_m(t))$. If the noise $w_m(t)$ is conditionally sub-Gaussian for all $t$: $\EE\sbr{\exp\rbr{ \lambda \cdot w_m(t+1)}} \le \exp\rbr{\lambda^2 \sigma^2/2} $, then with probability at least $1-\delta$, for all $t \ge 0$, we have:
\begin{align*}
    \sum_{m=1}^M \norm{\sum_{s=0}^t w_m(s+1) \bar U' x_m(s)}_{\bar V_m^{-1}(t)}^2 \le 2 \log \rbr{\frac{\Pi_{m=1}^M \rbr{\det (\bar V_m(t))}^{1/2} \rbr{\bar U' V_0 \bar U}^{-1/2}}{\delta}}
\end{align*}
\end{lem}

In order to use the above result in our case, we consider the martingale sum $\sum_{t=0}^T \sum_{j=1}^d \tilde{\eta}_{m,j}(t) \bar U'\tilde{x}_{m,j}(t)$. Under \pref{assum:noise}, we can use the same argument as in the proof of Lemma 2 in Hu et al. \cite{hu2021near} as:
\begin{align*}
    \exp{\rbr{\sum_{j=1}^d \frac{\eta_m(t+1)[j]}{\sigma}\inner{\lambda}{\bar{x}_{m,j}(t)} }} \le \exp{\rbr{\sum_{j=1}^d \frac{1}{2}\inner{\lambda}{\bar{x}_{m,j}(t)}^2 }}.
\end{align*}
Thus, for a fixed matrix $\bar U$ and $T\ge0$, with probability at least $1-\delta_U$,
\begin{align*}
    \sum_{m=1}^M \norm{S_{m}(T)}_{\bar{V}_m^{-1}(T)}^2 \le 2\sigma^2\log \rbr{\frac{\Pi_{m=1}^M\det(\bar{V}_m(T))\det (V_0)^{-1}}{\delta_U}}
\end{align*}
Finally, we take a union bound over the $\epsilon$-cover set of orthonormal matrices $\R^{d^2 \times 2k}$ to bound the total failure probability by $\abr{\Ncal_\epsilon}\delta_U = \rbr{\frac{6\sqrt{2k}}{\epsilon}}^{2d^2k}\delta_U$.

\subsection{Putting Things Together}
\label{app:final_proof}
We now use the bounds we have shown for each term before and give the final steps in the proof of \pref{thm:estimation_error} by using the error decomposition in \pref{lem:breakup} as follows:
From \pref{lem:breakup}, with $a$ we have: 
\begin{align*}
    \norm{\Xcal (W^*B^* - \hat{W} \hat{B})}_F^2\le {} & \sqrt{2\sum_{m=1}^M \norm{\tilde\eta_m^\top \tilde{X}_m \bar{U}}_{\bar{V}^{-1}_m}^2}\norm{\Xcal (W^*B^* - \hat{W} \hat{B})}_F + \sum_{m=1}^M \inner{\tilde\eta_m}{\tilde{X}_m (U-\bar{U}) r_m} \\
    {} & + \sqrt{\sum_{m=1}^M \norm{\tilde\eta_m^\top \tilde{X}_m \bar{U}}_{\bar{V}^{-1}_m}^2} \sqrt{2\sum_{m=1}^M \norm{\tilde{X}_m(\bar{U}-U)r_m}^2}.
\end{align*}
Now, let $\abr{\Ncal_{\epsilon}}$ be the cardinality of the $\epsilon$-cover of the set of orthonormal matrices in $\R^{d^2 \times 2k}$ that we defined in \pref{lem:breakup}. So, substituting the termwise bounds from \pref{prop:proj_cover_error}, \pref{prop:breakup_inner_prod}, and \pref{prop:mt-self-norm}, with probability at least $1-\abr{\Ncal_{\epsilon}}\delta_U - \delta_Z$, it holds that:
\begin{align}
    \frac{1}{2}\norm{\Xcal (W^*B^* - \hat{W} \hat{B})}_F^2 \lesssim {} & \sqrt{\sigma^2\log \rbr{\frac{\Pi_{m=1}^M\det(\bar{V}_m(t))\det (V_0)^{-1}}{\delta_U}}}\norm{\Xcal (W^*B^* - \hat{W} \hat{B})}_F \nonumber \\
    {} & + \sqrt{\sigma^2\log \rbr{\frac{\Pi_{m=1}^M\det(\bar{V}_m(t))\det (V_0)^{-1}}{\delta_U}}} \sqrt{\boldsymbol\kappa\epsilon^2 \rbr{MT\tr{C} +\sigma^2\log \frac{1}{\delta_Z} }} \nonumber \\
    {} & + \sqrt{\boldsymbol\kappa} \epsilon \rbr{MT \tr{C} + \sigma^2\log \frac{1}{\delta_Z}}. & \label{eq:app_inter_breakup}
\end{align}
For the matrix $V_0$, we now substitute $\Sigma = \underbar{\lambda}I_{d^2}$, which implies that $\det(V_0)^{-1} = \det(1/\underbar{\lambda}I_{2k}) = \rbr{1/\underbar{\lambda}}^{2k}$. Similarly, for the matrix $\bar{V}_m(T)$, we get $\det(\bar{V}_m(T)) \le \bar{\lambda}^{2k}$. Thus, substituting $\delta_U = 3^{-1}\delta\abr{\Ncal}_\epsilon^{-1}$ and $\delta_C=3^{-1}\delta$ in \pref{thm:cov_conc}, with probability at least $1-2\delta/3$, the upper-bound in \pref{prop:mt-self-norm} becomes:
\begin{align*}
    \sum_{m=1}^M \norm{\tilde\eta_m^\top \tilde{X}_m \bar{U}}_{\bar{V}^{-1}_m}^2 \le {} & \sigma^2\log \rbr{\frac{\Pi_{m=1}^M\det(\bar{V}_m(t))\det (V_0)^{-1}}{\delta_U}} \\
    \le {} & \sigma^2 \log \rbr{ \frac{\bar{\lambda}}{\underbar{\lambda}} }^{2Mk} + \sigma^2 \log \rbr{ \frac{18k}{\delta \epsilon} }^{2d^2k} \\
    \lesssim {} & \sigma^2Mk \log \boldsymbol\kappa_\infty + \sigma^2 d^2k \log \frac{k}{\delta \epsilon}.
\end{align*}
Substituting this in \pref{eq:app_inter_breakup} with $\delta_Z = \delta/3$, with probability at least $1-\delta$, we have:
\begin{align*}
    \frac{1}{2} \norm{\Xcal (W^*B^* - \hat{W} \hat{B})}_F^2 \lesssim {} & \sqrt{\sigma^2Mk \log \boldsymbol\kappa_\infty + \sigma^2 d^2k \log \frac{k}{\delta \epsilon}}\norm{\Xcal (W^*B^* - \hat{W} \hat{B})}_F \\
    {} & + \sqrt{\sigma^2Mk \log \boldsymbol\kappa_\infty + \sigma^2 d^2k \log \frac{k}{\delta \epsilon}} \sqrt{\boldsymbol\kappa\epsilon^2 \rbr{MT\tr{C} +\sigma^2\log \frac{1}{\delta} }} \\
    {} & + \sqrt{\boldsymbol\kappa} \epsilon \rbr{MT \tr{C} + \sigma^2\log \frac{1}{\delta}}  \\ %%%%%%%%%%%%
    \lesssim {} & \sqrt{c^2Mk \log \boldsymbol\kappa_\infty + c^2 d^2k \log \frac{k}{\delta \epsilon}}\norm{\Xcal (W^*B^* - \hat{W} \hat{B})}_F \\
    {} & + \sqrt{c^2Mk \log \boldsymbol\kappa_\infty + c^2 d^2k \log \frac{k}{\delta \epsilon}} \sqrt{\boldsymbol\kappa\epsilon^2 \rbr{c^2dMT +c^2\log \frac{1}{\delta}}} \\
    {} & + \sqrt{\boldsymbol\kappa} \epsilon \rbr{c^2dMT + c^2\log \frac{1}{\delta}}  .%%%%%%%%%%%%
\end{align*}
Noting that $k \le d$ and $\log \frac{1}{\delta} \lesssim d^2k\log \frac{k}{\delta \epsilon}$ for $\epsilon = \frac{k}{\sqrt{\boldsymbol\kappa}dT}$, we obtain:
\begin{align*}
    \frac{1}{2} \norm{\Xcal (W^*B^* - \hat{W} \hat{B})}_F^2 \lesssim {} & \rbr{\sqrt{c^2Mk \log \boldsymbol\kappa_\infty + c^2 d^2k \log \frac{k}{\delta \epsilon}} }\norm{\Xcal (W^*B^* - \hat{W} \hat{B})}_F  \\
    {} & + \sqrt{c^2Mk \log \boldsymbol\kappa_\infty + c^2 d^2k \log \frac{k}{\delta \epsilon}} \sqrt{\boldsymbol\kappa\epsilon^2 \rbr{c^2dMT +c^2d^2k\log \frac{k}{\delta\epsilon} }} \\
    {} & + \sqrt{\boldsymbol\kappa} \epsilon \rbr{c^2dMT + c^2d^2k\log \frac{k}{\delta\epsilon}} \\ %%%%%%%%%%%%
    \lesssim {} & \rbr{\sqrt{c^2Mk \log \boldsymbol\kappa_\infty + c^2 d^2k \log \frac{\boldsymbol\kappa dT}{\delta}}}\norm{\Xcal (W^*B^* - \hat{W} \hat{B})}_F  \\
    {} & + \sqrt{c^2Mk \log \boldsymbol\kappa_\infty + c^2 d^2k \log \frac{\boldsymbol\kappa dT}{\delta}} \sqrt{c^2\rbr{\frac{k^2M}{dT} + \frac{k^3}{T^2}\log \frac{\boldsymbol\kappa dT}{\delta}  }} \\
    {} & + c^2\rbr{Mk + \frac{dk^2}{T}\log \frac{\boldsymbol\kappa dT}{\delta}} \\ %%%%%%%%%%%%
    \lesssim {} & \rbr{\sqrt{c^2\rbr{Mk \log \boldsymbol\kappa_\infty + d^2k \log \frac{\boldsymbol\kappa dT}{\delta}}} }\norm{\Xcal (W^*B^* - \hat{W} \hat{B})}_F  \\
    {} & + c^2\rbr{Mk \log \boldsymbol\kappa_\infty + \frac{d^2k}{T}\log \frac{\boldsymbol\kappa dT}{\delta}}.
\end{align*}
The above quadratic inequality for the prediction error $\norm{\Xcal(W^*B^* - \hat W \hat B)}_F^2$ implies the following bound, which holds with probability at least $1-\delta$.
\begin{align*}
    \norm{\Xcal (W^*B^* - \hat{W} \hat{B})}_F^2 \lesssim c^2\rbr{Mk \log \boldsymbol\kappa_\infty + d^2k\log \frac{\boldsymbol\kappa dT}{\delta}}.
\end{align*}
Since the smallest eigenvalue of the matrix $\Sigma_m = \sum_{t=0}^T X_m(t) X_m(t)'$ is at least $\underbar \lambda$ (\pref{thm:cov_conc}), we can convert the above prediction error bound to an estimation error bound and get
\begin{align*}
    \norm{W^*B^* - \hat W \hat B}_F^2 \lesssim {} & \frac{c^2\rbr{Mk \log \boldsymbol\kappa_\infty + d^2k\log \frac{\boldsymbol\kappa dT}{\delta}}}{\underbar \lambda},
\end{align*}
which implies the following estimation error bound for the solution of \pref{eq:mt_opt}:
\begin{align*}
    \sum_{m=1}^M \norm{\hat A_m - A_m}_F^2 \lesssim \frac{c^2\rbr{Mk \log \boldsymbol\kappa_\infty + d^2k\log \frac{\boldsymbol\kappa dT}{\delta}}}{\underbar \lambda}.
\end{align*}

\section{Detailed Proof of the Estimation Error in \pref{thm:estimation_error_pert}}
\label{app:estimation-error_pert}
In this section, we provide a detailed proof of the average estimation error across the $M$ systems for the estimator in \pref{eq:mt_opt} in presence of misspecifications $D_m \in \R^{d \times d}$:
\begin{align*}
    A_m = \rbr{\sum_{i=1}^k \beta^*_m[i] W^*_i} + D_m, \quad \text{where }\norm{D_m}_F \le \zeta_m
\end{align*}
% In the presence of misspecifications, we have $\Delta \coloneqq \tilde \Theta^* - \hat{\Theta} = VR + D$ where $V \in O^{d^2 \times 2k}$ is an orthonormal matrix, $R \in \R^{2k \times M}$ and $D \in \R^{d^2 \times M}$ is the misspecification error. Similar to the analysis of \pref{thm:estimation_error}, we start by fact that $(\hat W, \hat B)$ minimize the squared loss in \pref{eq:mt_opt}. However,
Recall that, in this case, we get an additional term in the squared loss decomposition for $(\hat W, \hat B)$ which depends on the misspecifications $D_m$ as follows:
\begin{align}
\label{eq:app_sqloss_ineq_pert}
    \frac{1}{2} \sum_{m=1}^M \norm{ \tilde{X}_m (W^*\beta^*_m - \hat{W} \hat{\beta}_m) }_2^2 \le & {} \sum_{m=1}^M \inner{\tilde \eta_m}{\tilde{X}_m \rbr{\hat{W}\hat{\beta}_m - W^*\beta^*_m}} {+ \sum_{m=1}^M 2\inner{ \tilde{X}_m\tilde{D}_m}{ \tilde{X}_m \rbr{\hat{W}\hat{\beta}_m - W^*\beta^*_m}}}.
\end{align}
The error in the shared part, $\hat{W}\hat{\beta}_m - W^*\beta^*_m$, can still be rewritten as $Ur_m$ where $U \in \R^{d^2 \times 2k}$ is a matrix containing an orthonormal basis of size $2k$ in $\R^{d^2}$ and $r_m \in \R^{2k}$ is the system specific vector. We now prove the squared loss decomposition result stated in \pref{lem:breakup_pert_app}:
\paragraph*{Lemma} [Restatement of \pref{lem:breakup_pert_app}]
% \label{lem:breakup_pert_app}
Under the misspecified shared linear basis structure in \pref{eq:linear_model_pert}, for any fixed orthonormal matrix $\bar{U} \in \R^{d^2 \times 2k}$, the low rank part of the total squared error can be decomposed as follows:
\begin{align*}
    & \frac{1}{2}\sum_{m=1}^M \norm{\tilde{X}_m (W^*\beta^*_m - \hat{W} \hat{\beta}_m)}_F^2 & \nonumber\\ 
    \le {} & \sqrt{\sum_{m=1}^M \norm{\tilde\eta_m^\top \tilde{X}_m \bar{U}}_{\bar{V}^{-1}_m}^2}\sqrt{2\sum_{m=1}^M \norm{ \tilde{X}_m (W^*\beta^*_m - \hat{W} \hat{\beta}_m) }_2^2} + \sum_{m=1}^M \inner{\tilde\eta_m}{\tilde{X}_m (U-\bar{U}) r_m} & \nonumber \\
    {} & + \sqrt{\sum_{m=1}^M \norm{\tilde\eta_m^\top \tilde{X}_m \bar{U}}_{\bar{V}^{-1}_m}^2} \sqrt{2\sum_{m=1}^M \norm{\tilde{X}_m(\bar{U}-U)r_m}^2} {+ 2\sqrt{\bar \lambda}\bar{\zeta} \sqrt{\sum_{m=1}^M \norm{\tilde{X}_m \rbr{\hat{W}\hat{\beta}_m - W^*\beta^*_m}}_2^2}}. & 
    % \label{eq:app_pred_error_breakup_pert}
\end{align*}
\begin{proof}
Letting $\tilde{\Sigma}_{m,\up}$ and $\tilde{\Sigma}_{m,\dn}$ as defined in \pref{app:estimation-error}, recall that we define $V_m = U^\top \tilde{\Sigma}_m U + U^\top\tilde{\Sigma}_{m,\dn}U$ be the regularized covairance matrix of projected covariates $\tilde{X}_m U$. For any orthonormal matrix $\bar{U} \in \R^{d^2 \times 2k}$, we define $\bar V_m = \bar U^\top \tilde{\Sigma}_m \bar U + \bar U^\top\tilde{\Sigma}_{m,\dn}\bar U$ and proceed as follows:
\begin{align*}
    & \frac{1}{2}\sum_{m=1}^M \norm{ \tilde{X}_m (W^*\beta^*_m - \hat{W} \hat{\beta}_m) }_2^2 &\\
    \le {} & \sum_{m=1}^M \inner{\tilde \eta_m}{\tilde{X}_m \bar{U} r_m} + \sum_{m=1}^M \inner{\tilde\eta_m}{\tilde{X}_m (U-\bar{U}) r_m}  {+ \sum_{m=1}^M 2\inner{ \tilde{X}_m\tilde{D}_m}{ \tilde{X}_m \rbr{\hat{W}\hat{\beta}_m - W^*\beta^*_m}}} \\
    \le {} & \sum_{m=1}^M \norm{\tilde\eta_m^\top \tilde{X}_m \bar{U}}_{\bar{V}^{-1}_m}\norm{r_m}_{\bar{V}_m} + \sum_{m=1}^M \inner{\tilde\eta_m}{\tilde{X}_m (U-\bar{U}) r_m}  {+ \sum_{m=1}^M 2\norm{\tilde{X}_m\tilde{D}_m}_2 \norm{\tilde{X}_m \rbr{\hat{W}\hat{\beta}_m - W^*\beta^*_m}}_2} \\ 
    \le {} & \sqrt{\sum_{m=1}^M \norm{\tilde\eta_m^\top \tilde{X}_m \bar{U}}_{\bar{V}^{-1}_m}^2}\sqrt{\sum_{m=1}^M \norm{r_m}^2_{V_m}} + \sum_{m=1}^M \inner{\tilde\eta_m}{\tilde{X}_m (U-\bar{U}) r_m} \\
    {} & + \sqrt{\sum_{m=1}^M \norm{\tilde\eta_m^\top \tilde{X}_m \bar{U}}_{\bar{V}^{-1}_m}^2} \sqrt{\sum_{m=1}^M \rbr{\norm{r_m}_{\bar{V}_m} - \norm{r_m}_{V_m}}^2}  {+ 2\sqrt{\sum_{m=1}^M \norm{\tilde{X}_m\tilde{D}_m}_2^2} \sqrt{\sum_{m=1}^M \norm{\tilde{X}_m \rbr{\hat{W}\hat{\beta}_m - W^*\beta^*_m}}_2^2}}. 
\end{align*}
The first equality uses the fact that the error matrix is low rank upto a misspecification term. The first inequality follows by using Cauchy-Schwarz inequality. In the last step, we have used the sub-additivity of square root to rewrite the first term in two parts. Now, we can rewrite the error as:

\begin{align*}
    & \frac{1}{2}\sum_{m=1}^M \norm{ \tilde{X}_m (W^*\beta^*_m - \hat{W} \hat{\beta}_m) }_2^2 &\\
    \le {} & \sqrt{\sum_{m=1}^M \norm{\tilde\eta_m^\top \tilde{X}_m \bar{U}}_{\bar{V}^{-1}_m}^2}\sqrt{2\sum_{m=1}^M \norm{Ur_m}^2_{\tilde{\Sigma}_m}} + \sum_{m=1}^M \inner{\tilde\eta_m}{\tilde{X}_m (U-\bar{U}) r_m} \\
    {} & + \sqrt{\sum_{m=1}^M \norm{\tilde\eta_m^\top \tilde{X}_m \bar{U}}_{\bar{V}^{-1}_m}^2} \sqrt{\sum_{m=1}^M \norm{\tilde{X}_m(\bar{U}-U)r_m}^2}  {+ 2\sqrt{\bar \lambda_m \zeta_m^2} \sqrt{\sum_{m=1}^M \norm{\tilde{X}_m \rbr{\hat{W}\hat{\beta}_m - W^*\beta^*_m}}_2^2}} \\
    \le {} & \sqrt{\sum_{m=1}^M \norm{\tilde\eta_m^\top \tilde{X}_m \bar{U}}_{\bar{V}^{-1}_m}^2}\sqrt{2\sum_{m=1}^M \norm{ \tilde{X}_m (W^*\beta^*_m - \hat{W} \hat{\beta}_m) }_2^2} + \sum_{m=1}^M \inner{\tilde\eta_m}{\tilde{X}_m (U-\bar{U}) r_m} \\
    {} & + \sqrt{\sum_{m=1}^M \norm{\tilde\eta_m^\top \tilde{X}_m \bar{U}}_{\bar{V}^{-1}_m}^2} \sqrt{\sum_{m=1}^M \norm{\tilde{X}_m(\bar{U}-U)r_m}^2}  {+ 2\sqrt{\bar \lambda}\bar{\zeta} \sqrt{\sum_{m=1}^M \norm{\tilde{X}_m \rbr{\hat{W}\hat{\beta}_m - W^*\beta^*_m}}_2^2}}.
\end{align*}
\end{proof}
We will now bound each term individually. For the matrix $\bar{U}$, we choose it to be an element of $\Ncal_{\epsilon}$ which is an $\epsilon$-cover of the set of orthonormal matrices in $\R^{d^2 \times 2k}$. Therefore, for any $U$, there exists $\bar{U}$ such that $\norm{\bar{U} - U}_F \le \epsilon$. We can bound the size of such a cover using \pref{lem:lowrank-cover} as $|\Ncal_\epsilon| \le \rbr{\frac{6\sqrt{d}}{\epsilon}}^{2d^2k}$.

\paragraph*{Proposition} [Restatement of \pref{prop:proj_cover_error_pert}]
For the multi-task model specified in \pref{eq:linear_model_pert}, for the noise process $\{\eta_m(t)\}_{t=1}^\infty$ defined for each system, with probability at least $1-\delta_Z$, we have:
\begin{align*}
    \sum_{m=1}^M \norm{\tilde{X}_m(\bar{U}-U)r_m}^2 \lesssim {} & \boldsymbol\kappa\epsilon^2 \rbr{MT\tr{C} +\sigma^2\log \frac{2}{\delta_Z}  {+ \bar{\lambda}\bar{\zeta}^2} }. 
    % \label{eq:pred_error_coarse_pert}
\end{align*}
\begin{proof}
In order to bound the term above, we use the squared loss inequality in \pref{eq:app_sqloss_ineq_pert} and \pref{eq:linear_model_pert} as follows: 
\begin{align*}
    \norm{\Xcal (W^*B^* - \hat{W} \hat{B})}_F^2 \le {} & 2 \inner{Z}{\Xcal (W^*B^* - \hat{W} \hat{B})}  {+ 2\sum_{m=1}^M \inner{\tilde{X}_m\tilde{D}_m}{\tilde{X}_m \rbr{W^*\beta^*_m - \hat{W}\hat{\beta}_m}}}\\
    \le {} & 2\norm{Z}_F \norm{\Xcal (W^*B^* - \hat{W} \hat{B})}_F  {+ 2 \sqrt{\sum_{m=1}^M \norm{\tilde{X}_m\tilde{D}_m}_2^2} \norm{\Xcal (W^*B^* - \hat{W} \hat{B})}_F}\\
    \le {} & 2\norm{Z}_F \norm{\Xcal (W^*B^* - \hat{W} \hat{B})}_F  {+ 2 \sqrt{\sum_{m=1}^M \bar{\lambda}_m \norm{D_m}_F^2} \norm{\Xcal (W^*B^* - \hat{W} \hat{B})}_F} \\
    \le {} & 2\norm{Z}_F \norm{\Xcal (W^*B^* - \hat{W} \hat{B})}_F  {+ 2 \sqrt{\bar{\lambda} \bar \zeta^2} \norm{\Xcal (W^*B^* - \hat{W} \hat{B})}_F},
\end{align*}
which leads to the inequality $\norm{\Xcal (W^*B^* - \hat{W} \hat{B})}_F \le 2 \norm{Z}_F  {+ 2 \sqrt{\bar{\lambda} \bar \zeta^2}}$. Using the concentration result in \pref{prop:noise-magnitude}, with probability at least $1-\delta_Z$, we get
\begin{align*}
    \|Z\|_F \lesssim \sqrt{MT\tr{C} + \sigma^2\log \frac{1}{\delta_Z}}.
\end{align*}
Thus, we have $\norm{\Xcal (W^*B^* - \hat{W} \hat{B})}_F \lesssim 2 \sqrt{MT\tr{C} + \sigma^2\log \frac{1}{\delta_Z}}  {+ 2 \sqrt{\bar{\lambda} \bar \zeta^2}}$ with probability at least $1-\delta_Z$. We now use this to bound the initial term:
\begin{align*}
    \sum_{m=1}^M \norm{\tilde{X}_m(\bar{U}-U)r_m}^2 \le {} & \sum_{m=1}^M \norm{\tilde{X}_m}^2\norm{\bar{U}-U}^2 \norm{r_m}^2 \\
    \le {} & \sum_{m=1}^M \bar{\lambda}_m \epsilon^2 \norm{r_m}^2 \\
    = {} & \sum_{m=1}^M \bar{\lambda}_m \epsilon^2 \norm{Ur_m}^2 \\
    \le {} & \sum_{m=1}^M \bar{\lambda}_m \epsilon^2 \frac{\norm{\tilde{X}_m U r_m}^2}{\underbar{\lambda}_m}\\
    \le {} & \boldsymbol\kappa \epsilon^2 \sum_{m=1}^M \norm{\tilde{X}_m U r_m}^2 \\
    = {} & \boldsymbol\kappa \epsilon^2 \sum_{m=1}^M \norm{\Xcal (W^*B^* - \hat{W} \hat{B})}_F^2 \\
    \lesssim {} & \boldsymbol\kappa\epsilon^2 \rbr{MT\tr{C} +\sigma^2\log \frac{1}{\delta_Z}  {+ \bar{\lambda} \bar \zeta^2} }.
\end{align*}
\end{proof}

\paragraph*{Proposition} [Restatement of \pref{prop:breakup_inner_prod_pert}]
Under \pref{assum:noise} and the shared structure in \pref{eq:linear_model_pert}, with probability at least $1-\delta_Z$ we have:
\begin{align*}
% \label{eq:breakup_inner_prod_pert}
    \sum_{m=1}^M \inner{\tilde\eta_m}{\tilde{X}_m (U-\bar{U}) r_m} \lesssim {} & \sqrt{\boldsymbol\kappa} \epsilon \rbr{MT \tr{C} + \sigma^2\log \frac{1}{\delta_Z}}  {+\sqrt{\boldsymbol\kappa \bar\lambda} \sqrt{MT \tr{C} + \sigma^2\log \frac{1}{\delta_Z}} \epsilon \bar \zeta}.
\end{align*}
\begin{proof}
Using Cauchy-Schwarz inequality, we bound the term as follows:
\begin{align*}
    \sum_{m=1}^M \inner{\tilde\eta_m}{\tilde{X}_m (U-\bar{U}) r_m} \le {} & \sqrt{\sum_{m=1}^M \norm{\tilde{\eta}_m}^2} \sqrt{\sum_{m=1}^M \norm{\tilde{X}_m(\bar{U}-U)r_m}^2} \\
    \lesssim {} & \sqrt{MT \tr{C} + \sigma^2\log \frac{1}{\delta_Z}} \sqrt{\boldsymbol\kappa\epsilon^2 \rbr{MT\tr{C} +\sigma^2\log \frac{1}{\delta_Z}  {+ \bar{\lambda} \bar \zeta^2} }}\\
    \lesssim {} & \sqrt{\boldsymbol\kappa} \epsilon \rbr{MT \tr{C} + \sigma^2\log \frac{1}{\delta_Z}}  {+\sqrt{\boldsymbol\kappa \bar\lambda} \sqrt{MT \tr{C} + \sigma^2\log \frac{1}{\delta_Z}} \epsilon \bar \zeta}.
\end{align*}
\end{proof}

\subsection{Putting Things Together} We now use the bounds we have shown for each term before and give the final steps in the proof of \pref{thm:estimation_error_pert} by using the error decomposition in \pref{lem:breakup_pert_app} as follows:
\begin{proof}
From \pref{lem:breakup_pert_app}, we have: 
\begin{align*}
    & \norm{\Xcal (W^*B^* - \hat{W} \hat{B})}_F^2 & \\ 
    \le {} & \sqrt{2\sum_{m=1}^M \norm{\tilde\eta_m^\top \tilde{X}_m \bar{U}}_{\bar{V}^{-1}_m}^2}\norm{\Xcal (W^*B^* - \hat{W} \hat{B})}_F + \sum_{m=1}^M \inner{\tilde\eta_m}{\tilde{X}_m (U-\bar{U}) r_m} & \\
    {} & + \sqrt{\sum_{m=1}^M \norm{\tilde\eta_m^\top \tilde{X}_m \bar{U}}_{\bar{V}^{-1}_m}^2} \sqrt{\sum_{m=1}^M \norm{\tilde{X}_m(\bar{U}-U)r_m}^2}  {+ 2\sqrt{\bar \lambda}\bar{\zeta} \norm{\Xcal (W^*B^* - \hat{W} \hat{B})}_F}. &
\end{align*}
\end{proof}
Now, substituting the termwise bounds from \pref{prop:proj_cover_error_pert}, \pref{prop:breakup_inner_prod_pert} and \pref{prop:mt-self-norm}, with probability at least $1-\abr{\Ncal_{\epsilon}}\delta_U - \delta_Z$ we get:
\begin{align}
    & \frac{1}{2}\norm{\Xcal (W^*B^* - \hat{W} \hat{B})}_F^2 & \nonumber\\ 
    \lesssim {} & \sqrt{\sigma^2\log \rbr{\frac{\Pi_{m=1}^M\det(\bar{V}_m(t))\det (V_0)^{-1}}{\delta_U}}}\norm{\Xcal (W^*B^* - \hat{W} \hat{B})}_F  {+ \sqrt{\bar \lambda}\bar{\zeta} \norm{\Xcal (W^*B^* - \hat{W} \hat{B})}_F} & \nonumber \\
    {} & + \sqrt{\sigma^2\log \rbr{\frac{\Pi_{m=1}^M\det(\bar{V}_m(t))\det (V_0)^{-1}}{\delta_U}}} \sqrt{\boldsymbol\kappa\epsilon^2 \rbr{MT\tr{C} +\sigma^2\log \frac{1}{\delta_Z}  {+ \bar{\lambda}\bar{\zeta}^2} }} \nonumber \\
    {} & + \sqrt{\boldsymbol\kappa} \epsilon \rbr{MT \tr{C} + \sigma^2\log \frac{1}{\delta_Z}}  {+\sqrt{\boldsymbol\kappa \bar\lambda} \sqrt{MT \tr{C} + \sigma^2\log \frac{1}{\delta_Z}} \epsilon \bar \zeta}. & \label{eq:app_inter_breakup_pert}
\end{align}
In the definition of $V_0$, we now substitute $\Sigma = \underbar{\lambda}I_{d^2}$ thereby implying $\det(V_0)^{-1} = \det(1/\underbar{\lambda}I_{2k}) = \rbr{1/\underbar{\lambda}}^{2k}$. Similarly, for matrix $\bar{V}_m(T)$, we get $\det(\bar{V}_m(T)) \le \bar{\lambda}^{2k}$. Thus, substituting $\delta_U = \delta/3\abr{\Ncal}_\epsilon$ and $\delta_C=\delta/3$ (in \pref{thm:cov_conc}), with probability at least $1-2\delta/3$, we get:
\begin{align*}
    \sum_{m=1}^M \norm{\tilde\eta_m^\top \tilde{X}_m \bar{U}}_{\bar{V}^{-1}_m}^2 \le {} & \sigma^2\log \rbr{\frac{\Pi_{m=1}^M\det(\bar{V}_m(t))\det (V_0)^{-1}}{\delta_U}} \\
    \le {} & \sigma^2 \log \rbr{ \frac{\bar{\lambda}}{\underbar{\lambda}} }^{2Mk} + \sigma^2 \log \rbr{ \frac{18k}{\delta \epsilon} }^{2d^2k} \\
    \lesssim {} & \sigma^2Mk \log \boldsymbol\kappa_\infty + \sigma^2 d^2k \log \frac{k}{\delta \epsilon}.
\end{align*}
Substituting this in \pref{eq:app_inter_breakup_pert} with $\delta_Z = \delta/3$, with probability at least $1-\delta$ we have:
\begin{align*}
    \frac{1}{2} \norm{\Xcal (W^*B^* - \hat{W} \hat{B})}_F^2 \lesssim {} & \sqrt{\sigma^2Mk \log \boldsymbol\kappa_\infty + \sigma^2 d^2k \log \frac{k}{\delta \epsilon}}\norm{\Xcal (W^*B^* - \hat{W} \hat{B})}_F  {+ \sqrt{\bar \lambda}\bar{\zeta} \norm{\Xcal (W^*B^* - \hat{W} \hat{B})}_F}  \\
    {} & + \sqrt{\sigma^2Mk \log \boldsymbol\kappa_\infty + \sigma^2 d^2k \log \frac{k}{\delta \epsilon}} \sqrt{\boldsymbol\kappa\epsilon^2 \rbr{MT\tr{C} +\sigma^2\log \frac{1}{\delta}  {+ \bar{\lambda}\bar{\zeta}^2} }} \\
    {} & + \sqrt{\boldsymbol\kappa} \epsilon \rbr{MT \tr{C} + \sigma^2\log \frac{1}{\delta}}  {+\sqrt{\boldsymbol\kappa \bar\lambda} \sqrt{MT \tr{C} + \sigma^2\log \frac{1}{\delta}} \epsilon \bar \zeta} \\ %%%%%%%%%%%%
    \lesssim {} & \sqrt{c^2Mk \log \boldsymbol\kappa_\infty + c^2 d^2k \log \frac{k}{\delta \epsilon}}\norm{\Xcal (W^*B^* - \hat{W} \hat{B})}_F  {+ \sqrt{\bar \lambda}\bar{\zeta} \norm{\Xcal (W^*B^* - \hat{W} \hat{B})}_F}  \\
    {} & + \sqrt{c^2Mk \log \boldsymbol\kappa_\infty + c^2 d^2k \log \frac{k}{\delta \epsilon}} \sqrt{\boldsymbol\kappa\epsilon^2 \rbr{c^2dMT +c^2\log \frac{1}{\delta}  {+ \bar{\lambda}\bar{\zeta}^2} }} \\
    {} & + \sqrt{\boldsymbol\kappa} \epsilon \rbr{c^2dMT + c^2\log \frac{1}{\delta}}  {+\sqrt{\boldsymbol\kappa \bar\lambda} \sqrt{c^2 dMT + c^2\log \frac{1}{\delta}} \epsilon \bar \zeta} .%%%%%%%%%%%%
\end{align*}
Noting that $k \le d$ and $\log \frac{1}{\delta} \lesssim d^2k\log \frac{k}{\delta \epsilon}$, by setting $\epsilon = \frac{k}{\sqrt{\boldsymbol\kappa}dT}$ we have:
\begin{align*}
    \frac{1}{2} \norm{\Xcal (W^*B^* - \hat{W} \hat{B})}_F^2 \lesssim {} & \rbr{\sqrt{c^2Mk \log \boldsymbol\kappa_\infty + c^2 d^2k \log \frac{k}{\delta \epsilon}}  {+ \sqrt{\bar \lambda}\bar{\zeta}}}\norm{\Xcal (W^*B^* - \hat{W} \hat{B})}_F  \\
    {} & + \sqrt{c^2Mk \log \boldsymbol\kappa_\infty + c^2 d^2k \log \frac{k}{\delta \epsilon}} \sqrt{\boldsymbol\kappa\epsilon^2 \rbr{c^2dMT +c^2d^2k\log \frac{k}{\delta\epsilon}  {+ \bar{\lambda}\bar{\zeta}^2} }} \\
    {} & + \sqrt{\boldsymbol\kappa} \epsilon \rbr{c^2dMT + c^2d^2k\log \frac{k}{\delta\epsilon}}  {+ \sqrt{\rbr{c^2 dMT + c^2d^2k\log \frac{k}{\delta \epsilon}} \boldsymbol\kappa \bar \lambda \epsilon^2 \bar \zeta^2}} \\ %%%%%%%%%%%%
    \lesssim {} & \rbr{\sqrt{c^2Mk \log \boldsymbol\kappa_\infty + c^2 d^2k \log \frac{\boldsymbol\kappa dT}{\delta}}  {+ \sqrt{\bar \lambda}\bar{\zeta}}}\norm{\Xcal (W^*B^* - \hat{W} \hat{B})}_F  \\
    {} & + \sqrt{c^2Mk \log \boldsymbol\kappa_\infty + c^2 d^2k \log \frac{\boldsymbol\kappa dT}{\delta}} \sqrt{c^2\rbr{\frac{k^2M}{dT} + \frac{k^3}{T^2}\log \frac{\boldsymbol\kappa dT}{\delta}  {+ \frac{\bar\lambda k^2 \bar \zeta^2}{d^2T^2} }}} \\
    {} & + c^2\rbr{Mk + \frac{dk^2}{T}\log \frac{\boldsymbol\kappa dT}{\delta}}  {+ \sqrt{c^2\rbr{\frac{k^2M}{dT} + \frac{k^3}{T^2}\log \frac{\boldsymbol\kappa dT}{\delta}} \bar \lambda \bar \zeta^2}} \\ %%%%%%%%%%%%
    \lesssim {} & \rbr{\sqrt{c^2\rbr{Mk \log \boldsymbol\kappa_\infty + d^2k \log \frac{\boldsymbol\kappa dT}{\delta}}}  {+ \sqrt{\bar \lambda}\bar{\zeta}}}\norm{\Xcal (W^*B^* - \hat{W} \hat{B})}_F  \\
    {} & + c^2\rbr{Mk \log \boldsymbol\kappa_\infty + \frac{d^2k}{T}\log \frac{\boldsymbol\kappa dT}{\delta}}  {+ c \sqrt{\frac{\bar \lambda \bar \zeta^2}{T}\rbr{Mk \log \boldsymbol\kappa_\infty + \frac{d^2k}{T}\log \frac{\boldsymbol\kappa dT}{\delta}}}}.
\end{align*}
The quadratic inequality for the prediction error $\norm{\Xcal(W^*B^* - \hat W \hat B)}_F^2$ implies the following bound that holds with probability at least $1-\delta$:
\begin{align*}
    \norm{\Xcal (W^*B^* - \hat{W} \hat{B})}_F^2 \lesssim c^2\rbr{Mk \log \boldsymbol\kappa_\infty + d^2k\log \frac{\boldsymbol\kappa dT}{\delta}} +  {\bar \lambda \bar \zeta^2}.
\end{align*}
Since $\underbar \lambda = \min_m \underbar \lambda_m$, we can convert the prediction error bound to an estimation error bound as follows:
\begin{align*}
    \norm{W^*B^* - \hat W \hat B}_F^2 \lesssim {} & \frac{c^2\rbr{Mk \log \boldsymbol\kappa_\infty + d^2k\log \frac{\boldsymbol\kappa dT}{\delta}}}{\underbar \lambda} +  {\boldsymbol\kappa_\infty \bar \zeta^2},
\end{align*}
which finally implies the estimation error bound for the solution of \pref{eq:mt_opt}:
\begin{align*}
    \sum_{m=1}^M \norm{\hat A_m - A_m}_F^2 \lesssim \frac{c^2\rbr{Mk \log \boldsymbol\kappa_\infty + d^2k\log \frac{\boldsymbol\kappa dT}{\delta}}}{\underbar \lambda} +  {(\boldsymbol\kappa_\infty + 1) \bar \zeta^2}.
\end{align*}

\end{document}